\documentclass[10pt,journal,compsoc]{IEEEtran}

\ifCLASSOPTIONcompsoc
  \usepackage[nocompress]{cite}
\else
  \usepackage{cite}
\fi

\ifCLASSINFOpdf
  \usepackage[pdftex]{graphicx}

\else

  \usepackage[dvips]{graphicx}

\fi

\usepackage{amsmath}

\usepackage{algorithmic}

\usepackage{array}

\usepackage{url}

\usepackage{amssymb}
\usepackage{accents}
\usepackage{subcaption}
\usepackage{multirow, makecell}
\usepackage{booktabs}
\usepackage{tabularx}
\usepackage{rotating}
\usepackage{lipsum}
\usepackage{threeparttable}

\usepackage[dvipsnames]{xcolor}
\usepackage{capt-of}
\usepackage{wrapfig}
\usepackage{makecell}
\usepackage{bm}
\usepackage{arydshln}
\usepackage{placeins}
\usepackage{algorithm}
\usepackage{algorithmic}
\usepackage{amssymb,amsmath,amsthm}

\usepackage{amsmath}
\usepackage{amssymb}
\usepackage{mathtools}
\usepackage{amsthm}
\theoremstyle{plain}
\newtheorem{theorem}{Theorem}[section]
\newtheorem{lemma}[theorem]{Lemma}

\newtheorem{corollary}{Corollary}
\newtheorem{proposition}{Proposition}
\newtheorem{definition}{Definition}

\newtheorem{remark}{Remark}

\graphicspath{{figs/}}
\usepackage{mathtools}
\usepackage{amsthm}
\usepackage{rotating}
\usepackage{booktabs} 

\def\E{\mathbb{E}}
\def\R{\mathbb{R}}

\title{Fast Graph Generation via Spectral Diffusion}

\author{Tianze~Luo$^\dagger$,
        Zhanfeng~Mo$^\dagger$,
        Sinno~Jialin~Pan*
\IEEEcompsocitemizethanks{\IEEEcompsocthanksitem T. Luo, Z. Mo and S. J. Pan are with the School
of Computer Science and Engineering, Nanyang Technological University, Singapore. $\dagger$ indicates co-first authors with equal contribution. * indicates corresponding author.

E-mail: \{tianze001,zhanfeng001,sinnopan\}@ntu.edu.sg
}
}

\begin{document}

\IEEEtitleabstractindextext{%
\begin{abstract}
Generating graph-structured data is a challenging problem, which requires learning the underlying distribution of graphs. Various models such as graph VAE, graph GANs, and graph diffusion models have been proposed to generate meaningful and reliable graphs, among which the diffusion models have achieved state-of-the-art performance. In this paper, we argue that running full-rank  diffusion SDEs on the whole graph adjacency matrix space hinders diffusion models from learning graph topology generation, and hence significantly deteriorates the quality of generated graph data. To address this limitation, we propose an efficient yet effective Graph Spectral Diffusion Model (GSDM), which is driven by low-rank diffusion SDEs on the graph spectrum space. Our spectral diffusion model is further proven to enjoy a substantially stronger theoretical guarantee than standard diffusion models. Extensive experiments across various datasets demonstrate that, our proposed GSDM turns out to be the SOTA model, by exhibiting both significantly higher generation quality and much less computational consumption than the baselines.  
\end{abstract}

}
\maketitle

\section{Introduction}

\IEEEPARstart{L}{earning} to generate graph-structural data not only requires knowing the nodes' feature distribution, but also a deep understanding of the underlying graph topology, which is essential to modelling various graph instances, such as social networks \cite{yang2020factorizable,wang2019mcne}, molecule structures \cite{zang2020moflow,shi2020graphaf}, neural architectures \cite{lee2021rapid}, recommender systems \cite{liu2021interpretable}, etc. Conventional likelihood-based graph generative models, e.g. GraphGAN \cite{wang2019learning}, GraphVAE \cite{simonovsky2018graphvae} and GraphRNN \cite{you2018graphrnn}, have demonstrated great strength on graph generation tasks. In general, a likelihood-based model is designed to learn the likelihood function of the underlying graph data distribution, with which one can draw new samples with preserved graph properties from the distribution of interest. However, most likelihood-based generative models suffer from either limited quality of modeling graph structures, or considerable computational burden \cite{jo2022score}.

\begin{figure*}[htbp]
    \centering
    \includegraphics[width=0.915\textwidth]{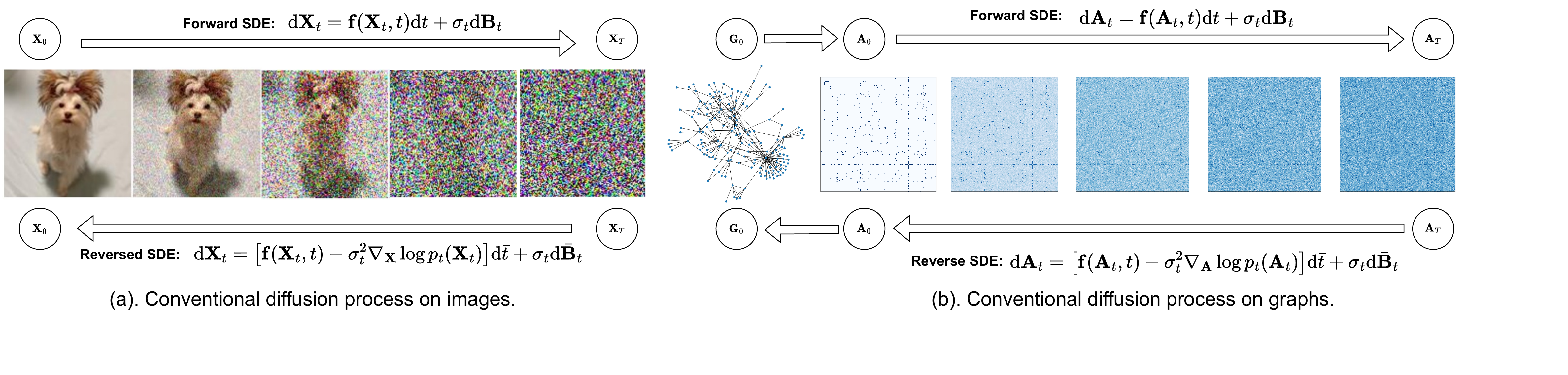}
    \caption{Illustation of the difference between applying the conventional SDE diffusion process on images (shown in (a)) and on graphs (shown in (b)). }
    \label{fig:intro}
\end{figure*}

\begin{figure*}
    \centering
    \begin{subfigure}[b]{0.275\textwidth}
         \centering
         \includegraphics[width=\textwidth]{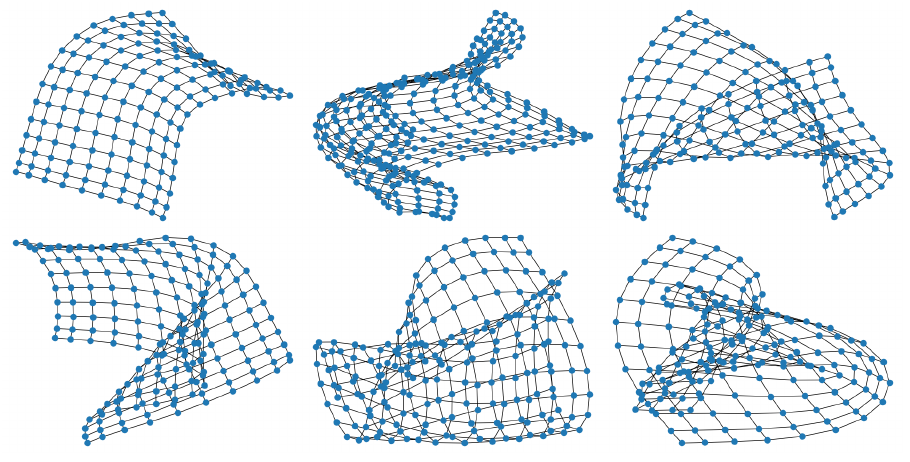}
         \caption{Real data}
         \label{fig:y equals x}
     \end{subfigure}
     \hfill
     \begin{subfigure}[b]{0.275\textwidth}
         \centering
         \includegraphics[width=\textwidth]{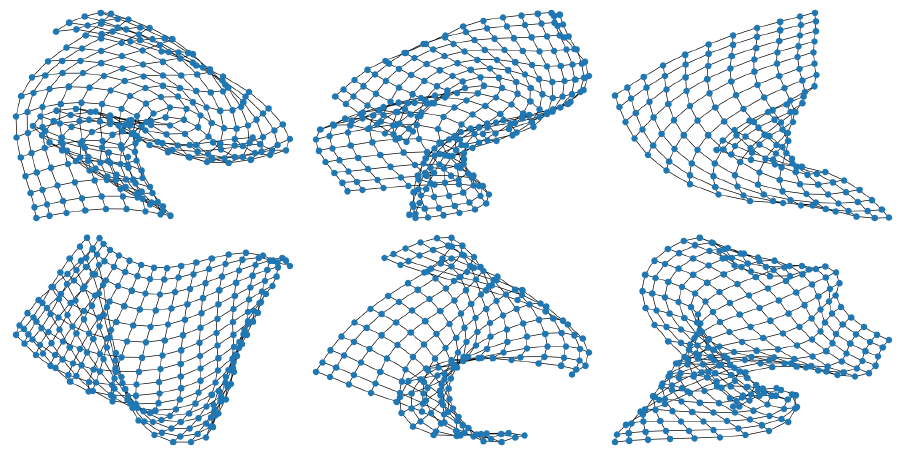}
         \caption{GSDM samples}
         \label{fig:three sin x}
     \end{subfigure}
     \hfill
     \begin{subfigure}[b]{0.275\textwidth}
         \centering
         \includegraphics[width=\textwidth]{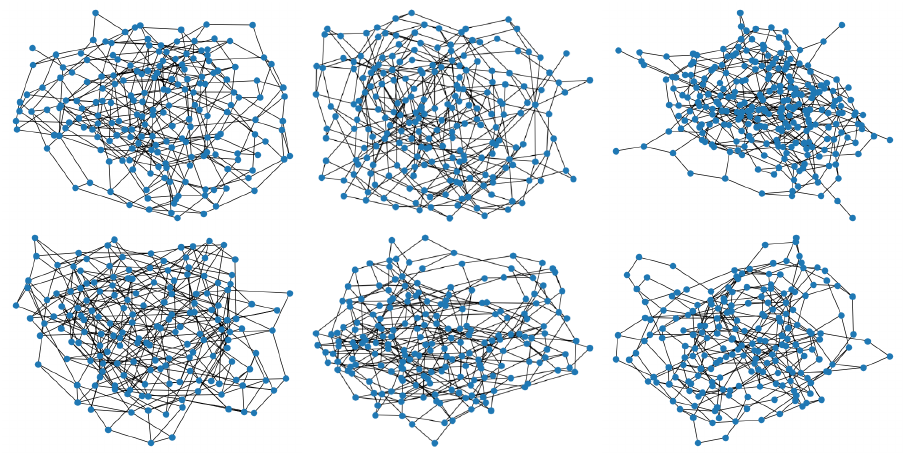}
         \caption{GDSS samples}
         \label{fig:five over x}
     \end{subfigure}
     \begin{subfigure}[b]{0.275\textwidth}
         \centering
         \includegraphics[width=\textwidth]{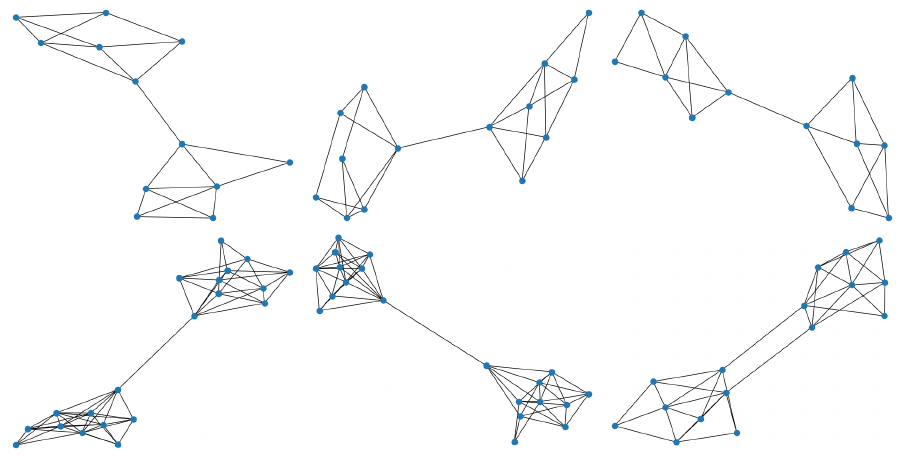}
         \caption{Real data}
         \label{fig:y equals x}
     \end{subfigure}
     \hfill
     \begin{subfigure}[b]{0.275\textwidth}
         \centering
         \includegraphics[width=\textwidth]{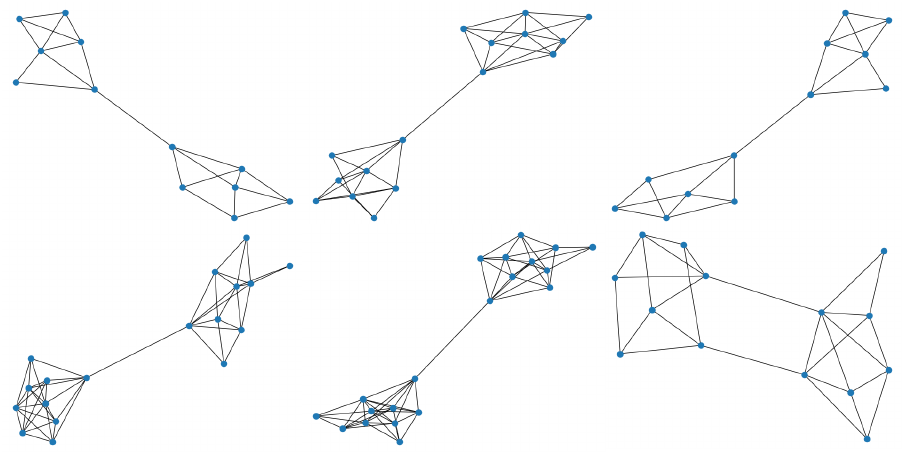}
         \caption{GSDM samples}
         \label{fig:three sin x}
     \end{subfigure}
     \hfill
     \begin{subfigure}[b]{0.275\textwidth}
         \centering
         \includegraphics[width=\textwidth]{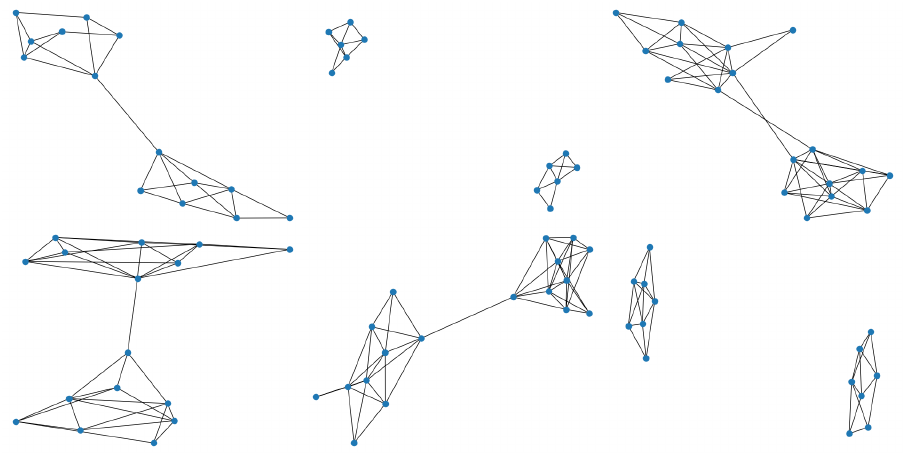}
         \caption{GDSS samples}
         \label{fig:five over x}
     \end{subfigure}
    \caption{Non-cherry-picked random samples from the testing set as well as samples generated by GSDM (ours) and GDSS \cite{jo2022score}, on Grid (top row) and Community-small (bottom row) datasets. For GDSS, we use the authors' released code and checkpoints to generate the samples.}
    \label{fig:gen}
\end{figure*}





Recently, a series of diffusion-based generative models have been proposed to overcome the limitations of likelihood-based models. Although being originally established for image generation \cite{song_ncsn}, diffusion models exhibit a great success in graph generation tasks with complex graph structural properties \cite{niu2020permutation,jo2022score}. Roughly speaking, diffusion refers to a Stochastic Differential Equation (SDE) that smoothly converts authentic data into pure noise via noise insertion. A diffusion model is able to reverse the data back from the noise, as long as the corresponding reversed-time SDE is known, which depends on the time-dependent gradient field of the density function (i.e. score function). To achieve this, a neural network is trained to learn the score function, and hence a reversed-time process can be constructed with the estimated score. The first graph diffusion model through SDEs, coined Graph Diffusion via SDE Systems (GDSS) \cite{jo2022score}, is designed to simultaneously generate node features and adjacency matrix via reversed diffusion. Similar to image diffusion models \cite{song_ncsn,song2020score}, at each diffusion step, GDSS directly inserts standard Gaussian noise to both node features and the adjacency matrix. Meanwhile, two separate neural networks are trained to learn the score functions of the node features and adjacency matrix, respectively. 

However, unlike the densely distributed image data, graph adjacency matrices can be highly sparse, which makes isotropic Gaussian noise insertion incompatible with graph structural data. In these circumstances, as shown in Figure \ref{fig:intro}, there is a stark difference between the diffusion process on images and on graph adjacency matrices. As can bee seen from the figure, the image corrupted by full-rank Gaussian noise exhibits recognizable numerical patterns along the early- and middle-stage of forward diffusion. However, the corrupted sparse graph adjacency matrix degenerates into a dense matrix with uniformly distributed entries in a few diffusion steps. In intuition, Figure \ref{fig:intro} implies that standard diffusion SDEs with full-rank isotropic noise insertion is destructive of learning graph topology and feature representations. Theoretically speaking, for extremely sparse graphs (e.g. molecules) with low-rank adjacency matrices, the adjacency score functions are supported on a low-dimensional manifold embedded in the full adjacency matrix space. Thus, directly applying diffusion models on graph topology generation is not desirable: once the diffusion SDE is run in the full space of the adjacency matrix, lethal noise will be injected into the out-of-support regions and drives the signal-to-noise ratio to be essentially zero, which is fatal for training score networks.

Even for densely connected graphs, the standard diffusion model is problematic for topology generation. Unlike image pixels that are merely locally correlated, an adjacency matrix governs the message-passing pattern of the whole graph. Thus, isotropic Gaussian noise insertion severely distorts the message-passing pattern, by blindly encouraging message passing on sparsely connected parts, which impedes the representation learning of sparse regions.









In order to establish a graph-friendly diffusion model, one should design an appropriate diffusion scheme that is compatible with the graph topology structure.
To this end, we propose the Graph Spectral Diffusion Model (GSDM), which is driven by diffusion SDEs on both the node feature space and the graph spectrum space. At each diffusion step, instead of corrupting the entire adjacency matrix, our method confines the Gaussian insertion to the graph spectrum space, i.e. the eigenvalue matrix of the adjacency matrix. This novel diffusion scheme enables us to perform smooth transformations on graph data during both the training and sampling phases. As illustrated in Figure \ref{fig:gen}, our proposed GSDM significantly outperforms the standard graph diffusion model (GDSS \cite{jo2022score}) in terms of graph generation quality and plausibility. For the Grid dataset shown in the top row of the figure, GDSS samples seem to be merely chaotic clusters, while GSDM samples exhibit smooth surface-like patterns that are visually similar to real data. For the Community-small dataset shown in the bottom row of the figure, GDSS fails to capture the link between two communities on some samples, while GSDM is able to capture dumbbell-like pattern (two clusters connected by one edge) as well as butterfly-like pattern (two clusters connected by two edges). This implies that GDSS's generation not only fails to mimic the observed topology distribution, but also suffers from capturing challenging details of the data such as links between communities. In contrast, GSDM is capable of generating high-quality graphs that are topologically similar to the real data, while retaining critical details.

We empirically evaluate the capability of our proposed GSDM on generic graph generation tasks, by evaluating the generation quality on both synthetic and real-world graph datasets. As shown in Section 4, GSDM outperforms existing one-shot generative models on various datasets, while achieving competitive performance to autoregressive models. Further molecule generation experiments show that our GSDM outperforms the state-of-the-art baselines, demonstrating that our proposed spectral diffusion model is capable of capturing complicated dependency between nodes and edges. Our main contributions are 3 folds:

\begin{itemize}
    \item We propose a novel Graph Spectral Diffusion Model (GSDM) for fast and quality graph generation. Our method overcomes the limitations of existing graph diffusion models by leveraging diffusion SDEs on both the node feature and graph spectrum spaces.

    \item Through the lens of stochastic analysis, we prove that GSDM enjoys a substantially stronger performance guarantee than the standard graph diffusion model. Our proposed spectral diffusion sharpens the reconstruction error bound from $\mathcal{O}(n^2\exp(n^2))$ to $\mathcal{O}(n\exp(n))$, where $n$ is the number of nodes.
    
    \item We evaluate GSDM on both synthetic and real-world graph generation tasks, on which GSDM outperforms all existing graph generative models. Moreover, GSDM also achieves evidently higher computational efficiency compared to existing graph diffusion models.

\end{itemize}
\section{Related Work}


From the view of graph generation strategies, graph generative models can be classified into two categories, sequential models and one-shot models. The sequential models, including GraphRNN \cite{you2018graphrnn}, GraphVAE \cite{simonovsky2018graphvae}, generate the nodes and edges in a sequential way, i.e. one after another \cite{guo2020systematic}, with validity checks among the generation steps. Models using a similar generation process are also known as autoregressive models. In contrast, one-shot generative models, e.g. GAN-based models \cite{de2018molgan}, VAE-based models \cite{ma2018constrained}, flow-based models \cite{zang2020moflow} and score-based models \cite{niu2020permutation,jo2022score}, learn the intrinsic dependency of graph structure by treating the distribution of all the components of a graph as a whole and therefore generate the entire graph in an integrative way. As a result, one-shot generative models enjoy the node permutation-invariant property. Moreover, one-shot models usually exhibit higher computational efficiency than autoregressive models. Our proposed GSDM generates data by reversing a spectral diffusion SDE with a learned spectral score function, thus it is essentially a score-based model. Since GSDM accepts the destination of the reversed diffusion SDE as the ultimately generated samples without any additional refinements, it can also be categorized as a one-shot model.




A recently proposed score-based model, GDSS \cite{jo2022score}, is the first and state-of-the-art diffusion-based generative model that simultaneously conducts nodes and edges generation. In essence, GDSS recasts the image diffusion paradigm \cite{song2020score} for graph generation. During the forward diffusion process, GDSS injects Gaussian noise to both the node features and the adjacency matrix at each diffusion step. Then, a neural network is trained to learn the score function by minimizing the score-matching objective, which enables a reversion of graph data from noise via a reversed time diffusion process. However, such a directly borrowed diffusion model is incompatible with graph topology generation: unlike images which are feature-rich, the graph adjacency matrix is  generally sparse and low-rank. Hence, injecting isotropic Gaussian noise into the sparsely connected parts of the adjacency matrix severely harms the graph data distribution and makes it hard to be recovered from Gaussian noise.



\section{Preliminaries}

\subsection{Notations}

In this paper, we denote the probability space of interest as $(\Omega, \mathcal{F}, \mathbb{P})$ and  $(\mathcal{F}_t)_{t\in \mathbb{R}}$ be a filtration, i.e. a sequence of increasing sub-$\sigma$-algebra of $\mathcal{F}$. Without specification, we denote $(\mathbf{B}_t)_{t\in \mathbb{R}}$ as the $d$-dimensional standard Brownian motion on the filtered probabilistic space $(\Omega, \mathcal{F}, \mathbb{P}, (\mathcal{F}_t)_{t\in \mathbb{R}})$. The distribution, support set and expectation of a random variable $\mathbf{z}$ are defined as $\mathrm{law}(\mathbf{z})$, $\mathrm{supp}(\mathbf{z})$ and $\mathbb{E}[\mathbf{z}]$. $\mathbf{y}|\mathbf{z}$ denotes the distribution of $\mathbf{y}$ conditioned on $\mathbf{z}$. $\mathrm{Unif}(A)$ denotes the uniform distribution on a set $A$. $\|\cdot \|$ denotes the standard Euclid norm. $\|\cdot\|_{\infty}$ and $\|\cdot\|_{\mathrm{lip}}$ denotes the supremum norm and Lipschitz norm of a function. $ [ \cdot ]$ denotes the flooring function.

\subsection{Score-based Generative Diffusion Model}

A generative model refers to a mapping $g_{\bm{\theta}}: \mathbb{R}^d\mapsto \mathbb{R}^d$, which maps a simple known priori $\pi$ to a complicated data distribution $\mathcal{D}$. Once the model $g_{\bm{\theta}}$ is sufficiently trained on $N$ i.i.d samples from $\mathcal{D}$, denoted by $\mathcal{S}$, it enables us to generate plausible instances from $\mathcal{D}$ directly, by sampling from $g_{\bm{\theta}}(\bm{\varepsilon}),\ \bm{\varepsilon} \sim \pi$. Unlike conventional generative models, e.g. VAE and GAN, which treat $\mathcal{D}$ as a unilateral transformation of $\pi$, diffusion models consider the bilateral relation between $\mathcal{D}$ and $\pi$ from the perspective of SDE. Given an SDE travelling from $\mathcal{D}$ to $\pi$, the corresponding reversed time SDE enables us to backtrack from noisy priori to the distribution of interest.

\begin{lemma}[Forward Diffusion and Reversed Time SDE \cite{reversed_time_sde_anderson}]\label{lem_reversed_sde}
    ~\\
    The Forward Diffusion refers to the following SDE
    \begin{align} 
        \mathbf{z}_0
        \sim
        \mathcal{D},\
        \mathrm{d} \mathbf{z}_t
        = 
        \mathbf{f}(\mathbf{z}_t, t) \mathrm{d}t
        +
        \sigma_t \mathrm{d}\mathbf{B}_t,\
        t\in [0,1], \label{fwd_sde}
    \end{align}
    where $\mathbf{f}(\cdot, t): \mathbb{R}^d \mapsto \mathbb{R}^d$ is the drift function, $\sigma_{t}: [0,1] \mapsto \mathbb{R}$ be a scalar diffusion function. Let $p_t(\cdot)$ be the probability density function of $\mathbf{z}_t$, then the Reversed Time SDE is given by 
    \begin{align}
    \label{reversed_sde}
        \mathrm{d} \bar{\mathbf{z}}_t
        = &
        (
            \mathbf{f}(\bar{\mathbf{z}}_t,t)
            -
            \sigma_t^2 \nabla \log p_t(\bar{\mathbf{z}}_t)
        )
        \mathrm{d} \bar{t}
        +
        \sigma_t \mathrm{d} \bar{\mathbf{B}}_t,
        \\
        \notag
        \bar{\mathbf{z}}_1 
        \sim &
        \mathbf{z}_1,\ 
        t \in [0,1],  
    \end{align}
    where $\mathrm{d}\bar{t}=-\mathrm{d}t$ is the negative infinitesimal time step,  $(\bar{\mathbf{B}}_t)_{t\in \mathbb{R}}$ is a reversed time Brownian motion w.r.t $(\Omega, \mathcal{F}, \mathbb{P}, (\bar{\mathcal{F}}_t)_{t\in \mathbb{R}})$, and $(\bar{\mathcal{F}}_t)_{t\in \mathbb{R}}$ is the corresponding decreasing filtration; $\nabla \log p_t(\cdot)$ is the score function.
\end{lemma}

During the forward diffusion process, with a carefully designed $\mathbf{f}(\cdot, t)$, the original data is perturbed by Gaussian noise with increasing magnitude, and it is assumed to be gradually corrupted to a truly noisy signal (priori), i.e. $\mathrm{law}(\mathbf{z}_1)=\pi$. In order to draw new data from $\pi$ via the reversed time SDE, one needs to learn the unknown score function $\nabla \log p_t(\cdot)$ with a neural network $s_{\bm{\theta}}(\cdot): \mathbb{R}^d \mapsto \mathbb{R}^d$, by minimizing the following explicit score matching error:
\begin{align}
\label{sm_objective}
    \mathcal{E}(\bm{\theta})
    \triangleq
    \E_{
    \mathbf{z}\sim \mathcal{D}
    } 
    \E_{\mathbf{z}_t|\mathbf{z}}
    \|
        s_{\bm{\theta}}(\mathbf{z}_t)
        -
        \nabla 
        \log p_t(\mathbf{z}_t) 
    \|^2.
\end{align}
In practice, we employ a Gaussian priori $\pi$. For each sample $\mathbf{z}^i\in \mathcal{S}$, we first generate a sequence of corrupted data $\{ \mathbf{z}^i_{t_j}\}_{j=1}^T$ by discretizing \eqref{fwd_sde}. To learn the score network $s_{\bm{\theta}}(\cdot)$, one can minimize a more tractable denoising score matching objective $\hat{\mathcal{E}}(\bm{\theta})$ 
\begin{align}
    \hat{\mathcal{E}}(\bm{\theta})
    \triangleq
    \E_{\mathbf{z}\sim \mathrm{Unif}(\mathcal{S})}
    \E_{\mathbf{z}_t | \mathbf{z}}
    \|
        s_{\bm{\theta}}(\mathbf{z}_t)
        -
        \nabla \log p(\mathbf{z}_t|\mathbf{z})
    \|^2,
    \label{emp_sm_objective}
\end{align}
which has been proven to be equivalent to $\mathcal{E}(\bm{\theta})$ in \cite{score_matching_tech}. Given a well trained score network $s_{\bm{\theta}^*}(\cdot)$, one is able to generate plausible data from $\pi$ via solving the learned reversed-time SDE 
\begin{align}
    \mathrm{d} \hat{\mathbf{z}}_t
    = &
    (
        \mathbf{f}(\hat{\mathbf{z}}_t,t)
        -
        \sigma_t^2 s_{\bm{\theta}^*}(\hat{\mathbf{z}}_t)
    )
    \mathrm{d} \bar{t}
    +
    \sigma_t \mathrm{d} \bar{\mathbf{B}}_t,
    \\
    \notag
    \hat{\mathbf{z}}_1 
    \sim &
    \pi,\
    t \in [0,1].
\end{align}
Ideally, the learned reversed time SDE should lead us towards $\mathcal{D}$, i.e. $\mathrm{law}(\hat{\mathbf{z}}_0) = \mathcal{D}$.

\section{Methodology}

In this section, we establish the Graph Spectral Diffusion Model (GSDM) for fast and effective graph data generation. In Section 4.1, we briefly review the standard score-based graph diffusion model \cite{jo2022score}. In Section 4.2, we formally introduce our GSDM algorithm and its $\alpha$-quantile variants. In Section 4.3, we provide theoretical analyses to justify the efficacy of GSDM on graph data generation. 

\subsection{Standard Graph Diffusion Model}

A graph with $n$ nodes is defined as $\mathbf{G} \triangleq (\mathbf{X}, \mathbf{A}) \in \R^{n\times d} \times \R^{n\times n}$, where $\mathbf{X} \in \R^{n \times d}$ is the node feature matrix with dimension $d$ and $\mathbf{A} \in \R^{n\times n}$ denotes the adjacency matrix. A graph generative model aims to learn the underlying data distribution, say $\mathbf{G} \sim \mathcal{G}$, which is a joint distribution of both $\mathbf{X}$ and $\mathbf{A}$. Note that, if $(\mathbf{X}, \mathbf{A})$ is treated as a whole and omit the intrinsic graph structure, the aforementioned score-based generation framework can be parallelly extended to the graph generation setting, which yields the standard graph diffusion model, i.e. GDSS \cite{jo2022score}. Roughly speaking, for each graph sample $(\mathbf{X},\mathbf{A})$, we first generate a sequence of perturbed graphs $\{(\mathbf{X}_{t_i}, \mathbf{A}_{t_i})\}_{i=1}^T$ via forward diffusion. Then, we train two score networks $s_{\bm{\theta}}(\cdot)$ and $s_{\bm{\phi}} (\cdot)$ to learn the score functions for both $\mathbf{X}_t$ and $\mathbf{A}_t$, with which we can generate new data from $\pi$ by running the reversed time SDE. Graph diffusion SDEs are defined as follows.

\begin{definition}[Graph Diffusion SDEs with Disentangled Drift]
~\\
    The Forward Graph Diffusion refers to the following SDE system
    \begin{align}
        &
        \label{graph_fwd_sde}
        \begin{cases}
            \mathrm{d} \mathbf{X}_t
            = &
            \mathbf{f}^X(\mathbf{X}_t, t) \mathrm{d}t
            +
            \sigma_{X,t} \mathrm{d} \mathbf{B}_t^X,
            \\
            \mathrm{d} \mathbf{A}_t
            = &
            \mathbf{f}^A(\mathbf{A}_t, t) \mathrm{d}t
            +
            \sigma_{A,t} \mathrm{d} \mathbf{B}_t^A,
        \end{cases}
        \\
        \notag
        & (\mathbf{X}_0, \mathbf{A}_0)
        \sim 
        \mathcal{G},\
        t\in [0,1],
    \end{align}
    where $\mathbf{f}^X(\cdot,t): \R^{n\times d} \mapsto \R^{n\times d}$ and $\mathbf{f}^A(\cdot,t): \R^{n\times n} \mapsto \R^{n\times n}$ is the drift functions for nodes feature and adjacency matrix; $\sigma_{X,t}, \sigma_{A,t}$ are the scalar diffusion terms (a.k.a noise schedule function); $(\mathbf{B}_t^X)_{t\in\R}$ and $(\mathbf{B}_t^A)_{t\in\R}$ are standard Brownian motions on $\mathbb{R}^{n\times d}$ and $\R^{n\times n}$, respectively.
\end{definition}

To alleviate the computational burden of calculating the drift w.r.t the high dimensional $\mathbf{G}$, the drift term of forward graph diffusion is disentangled into $\mathbf{f}^X(\cdot,t)$ and $\mathbf{f}^A(\cdot,t)$. Again, Lemma \ref{lem_reversed_sde} guarantees the existence of the reversed time SDEs for graph diffusion.

\begin{corollary}[Reversed Time SDEs for Graph Diffusion]
    The reversed time SDE system of \eqref{graph_fwd_sde} is given by
    \small
    \begin{align}
        &
        \notag
        \begin{cases}
            \mathrm{d} \bar{\mathbf{X}}_t
            = &
            \left(
                \mathbf{f}^X( \bar{\mathbf{X}}_t, t) 
                -
                \sigma_{X,t}^2
                \nabla_{\mathbf{X}}
                \log 
                p_t(\bar{\mathbf{X}}_t, \bar{\mathbf{A}}_t)
            \right)
            \mathrm{d}\bar{t}
            +
            \sigma_{X,t} \mathrm{d} \bar{\mathbf{B}}_t^X,
            \\
            \mathrm{d} \bar{\mathbf{A}}_t
            = &
            \left(
                \mathbf{f}^A( \bar{\mathbf{A}}_t, t) 
                -
                \sigma_{A,t}^2
                \nabla_{\mathbf{A}}
                \log 
                p_t(\bar{\mathbf{X}}_t, \bar{\mathbf{A}}_t)
            \right)
            \mathrm{d}\bar{t}
            +
            \sigma_{A,t} \mathrm{d} \bar{\mathbf{B}}_t^A,
        \end{cases}
        \\
        \label{graph_reversed_time_sde}
        &(\bar{\mathbf{X}}_1, \bar{\mathbf{A}}_1)
        \sim 
        \pi,\
        t \in [0,1],
    \end{align}
    \normalsize
    where $\mathrm{d}\bar{t}=-\mathrm{d}t$ is the negative infinitesimal time step; $(\bar{\mathbf{B}}_t^X)_{t\in\R}$ and $(\bar{\mathbf{B}}_t^A)_{t\in\R}$ are the reversed time standard Brownian motions induced by \eqref{graph_fwd_sde}.
\end{corollary}
 
The disentanglement of the drift functions implies the conditional independence $\mathbf{X}_t \perp \mathbf{A}_t | \mathbf{G}_0$, with which we can decompose $p_t(\bar{\mathbf{X}}_t, \bar{\mathbf{A}}_t)$ into $p_{t|0}(\mathbf{X}_t|\mathbf{X}_0) \cdot p_{t|0}(\mathbf{A}_t|\mathbf{A}_0)$, where $p_{t|0}(\cdot)$ denotes the density function of $\mathbf{G}_t|\mathbf{G_0}$. As proposed in \cite{jo2022score}, such conditional independence reduces the denoising score matching objective to a simpler form 
\small
\begin{align}
    \widehat{\mathcal{E}}(\bm{\theta})
    \triangleq &
    \E_{\mathbf{G}\sim \mathrm{Unif}(\mathcal{S})}
    \E_{\mathbf{G}_t | \mathbf{G}}
    \|
        s_{\bm{\theta}}(\mathbf{G}_t)
        -
        \nabla \log p_{t|0}(\mathbf{X}_t|\mathbf{X}_0)
    \|^2,
    \\
    \widehat{\mathcal{E}}(\bm{\phi})
    \triangleq &
    \E_{\mathbf{G}\sim \mathrm{Unif}(\mathcal{S})}
    \E_{\mathbf{G}_t | \mathbf{G}}
    \|
        s_{\bm{\phi}}(\mathbf{G}_t)
        -
        \nabla \log p_{t|0}(\mathbf{A}_t|\mathbf{A}_0)
    \|^2.
\end{align}
\normalsize
Hence, the training and sampling procedures can be directly borrowed from standard score-based models.

While GDSS is the first attempt at leveraging diffusion models on graph generation, its performance is hindered by the brute force application of diffusion. For sparsely connected graphs, while the distribution of node features varies across datasets, the distribution of graph topology, i.e. adjacency matrix, resides in a low dimensional manifold. As mentioned in Section 1, an evident pattern of the adjacency matrices also implies that the true distribution of $\mathbf{A}$ is of low rank. In this case, the score matching objective fails to provide consistency estimators. Although running a full rank diffusion on $\mathbf{A} \in \R^{n\times n}$ alleviates this issue by extending the support of corrupted data from the manifold to the full space, it inevitably introduces lethal noise to regions of zero probability density. As a consequence, the signal-to-noise ratio of regions out of $\mathrm{supp}(\mathbf{A})$ is essentially zero, which is a catastrophe for training the denoising score network. Note that such full-rank diffusion is also inappropriate for the densely connected graph generation. This is because an isotropically corrupted adjacency matrix encourages delusive message passing on sparsely connected parts of the graph, which is destructive to the graph message passing pattern. Thus, standard diffusion can severely impair representation learning for sparse graph regions.

\subsection{Graph Spectral Diffusion Model}

To address these notorious yet ubiquitous issues, we novelly propose the Graph Spectral Diffusion Model. For graph topology generation, in contrast to GDSS which is driven by a full-rank diffusion on the whole space $\R^{n\times n}$, our GSDM leverages low-rank diffusion SDEs on the $n$-dimensional spectrum manifold, e.g. the span of $n$-eigenvalues of $\mathbf{A}$. As we shall see later, GSDM achieves both robustness and computational efficiency, by exploiting the graph spectrum structure and running diffusion on an information-concentrated manifold. 
\begin{definition}[Graph Spectral Diffusion SDEs]
~\\
    Let the spectral decomposition of $\mathbf{A}$ be $\mathbf{U} \bm{\Lambda} \mathbf{U}^\top$, where columns of $\mathbf{U}$ are the orthonormal eigenvectors and $\bm{\Lambda}$ be the diagonal eigenvalue matrix, i.e. spectrum.
    The Forward Spectral Diffusion refers to the following SDE system
    \begin{align}
        &
        \label{spectral_fwd_sde}
        \begin{cases}
            \mathrm{d} \mathbf{X}_t
            = &
            \mathbf{f}^X(\mathbf{X}_t, t) \mathrm{d}t
            +
            \sigma_{X,t} \mathrm{d} \mathbf{B}_t^X,
            \\
            \mathrm{d} \bm{\Lambda}_t
            = &
            \mathbf{f}^{\Lambda}(\bm{\Lambda}_t, t) \mathrm{d}t
            +
            \sigma_{\Lambda,t} \mathrm{d} \mathbf{W}_t^{\Lambda},
        \end{cases}
        \\
        \notag
        & (\mathbf{X}_0, \mathbf{A}_0)
        \sim 
        \mathcal{G},\
        \mathbf{A}_0
        =
        \mathbf{U}_0
        \bm{\Lambda}_0
        \mathbf{U}_0^\top,\
        t\in [0,1],
    \end{align}
    where $\mathbf{f}^X(\cdot,t): \R^{n\times d} \mapsto \R^{n\times d}$ is the drift for nodes feature; $\mathbf{f}^{\Lambda}(\cdot,t): \R^{n} \mapsto \R^{n}$ is the drift for spectrum, which only acts on the diagonal entries; $\sigma_{X,t}, \sigma_{\Lambda,t}$ are the scalar diffusion terms (a.k.a noise schedule functions); $(\mathbf{B}_t^X)_{t\in\R}$ and $(\mathbf{B}_t^{\Lambda})_{t\in\R}$ are standard Brownian motions on $\mathbb{R}^{n\times d}$ and $\R^{n}$, respectively; and $\mathbf{W}_t^{\Lambda} \triangleq \mathrm{diag}(\mathbf{B}_t^{\Lambda})$ is a diagonal Brownian motion.
\end{definition}

By comparing \eqref{spectral_fwd_sde} with \eqref{graph_fwd_sde}, one should notice that the initial eigenvector matrix $\mathbf{U}_0$ is fixed along the spectral diffusion. As will be illustrated later, the evolution of $\mathbf{A}_t \triangleq \mathbf{U}_0 \bm{\Lambda}_t \mathbf{U}_0^\top$ is driven by a $n$-dimensional Gaussian process. Hence, we prevent the corrupted adjacency matrix from rampaging around the full space. We are now ready to establish the reversed time spectral diffusion SDEs.
\begin{corollary}[Reversed Time Spectral Diffusion SDEs]
    The reversed time Spectral Diffusion SDE system of \eqref{spectral_fwd_sde} is given by
    \small
    \begin{align}
        &
        \notag
        \begin{cases}
            \mathrm{d} \bar{\mathbf{X}}_t
            = &
            \left(
                \mathbf{f}^X( \bar{\mathbf{X}}_t, t) 
                -
                \sigma_{X,t}^2
                \nabla_{\mathbf{X}}
                \log 
                p_t(\bar{\mathbf{X}}_t, \bar{\bm{\Lambda}}_t)
            \right)
            \mathrm{d}\bar{t}
            +
            \sigma_{X,t} \mathrm{d} \bar{\mathbf{B}}_t^X,
            \\
            \mathrm{d} \bar{\bm{\Lambda}}_t
            = &
            \left(
                \mathbf{f}^{\Lambda}( \bar{\bm{\Lambda}}_t, t) 
                -
                \sigma_{\Lambda,t}^2
                \nabla_{\bm{\Lambda}}
                \log 
                p_t(\bar{\mathbf{X}}_t, \bar{\bm{\Lambda}}_t)
            \right)
            \mathrm{d}\bar{t}
            +
            \sigma_{\Lambda,t} \mathrm{d} \bar{\mathbf{W}}_t^{\Lambda},
        \end{cases}
        \\
        \label{spectral_reversed_time_sde}
        &(\bar{\mathbf{X}}_1, \bar{\bm{\Lambda}}_1)
        \sim 
        \pi,\
        t \in [0,1],
    \end{align}
    \normalsize
    where $\mathrm{d}\bar{t}=-\mathrm{d}t$ is the negative infinitesimal time step; $(\bar{\mathbf{B}}_t^X)_{t\in\R}$ and $(\bar{\mathbf{B}}_t^{\Lambda})_{t\in\R}$ are reversed time standard Brownian motions induced by \eqref{graph_fwd_sde}; and $\bar{\mathbf{W}}_t^{\Lambda} \triangleq \mathrm{diag}(\bar{\mathbf{B}}_t^{\Lambda})$.
\end{corollary}

Since $\mathbf{U}$ is no longer involved in \eqref{spectral_reversed_time_sde}, the boundary condition
is only imposed on the joint distribution of $(\mathbf{X}_1,\bm{\Lambda}_1)$ such that $\mathrm{law}(\mathbf{X}_1,\bm{\Lambda}_1)=\pi$. This assumption implies that, the authentic distribution $(\mathbf{X}_0,\bm{\Lambda}_0)$ can be recovered from a priori $\pi$ by the reversed time spectral diffusion SDE. According to score matching techniques \cite{score_matching_tech}, we can train two score networks $s_{\bm{\theta}}(\cdot,\cdot), s_{\bm{\phi}}(\cdot,\cdot)$ to learn the score functions $\nabla_{\mathbf{X}} \log p_t(\cdot, \cdot), \nabla_{\bm{\Lambda}} \log p_t(\cdot, \cdot)$ via minimizing 
\small
\begin{align}
\notag
    \widehat{\mathcal{E}}(\bm{\theta})
    \triangleq &
    \E_{\mathbf{G}\sim \mathrm{Unif}(\mathcal{S})}
    \E_{\mathbf{X}_t | \mathbf{G}}
    \|
        s_{\bm{\theta}}(\mathbf{X}_t, \bm{\Lambda}_t)
        -
        \nabla \log p_{t|0}(\mathbf{X}_t|\mathbf{X}_0)
    \|^2,
    \\
\notag
    \widehat{\mathcal{E}}(\bm{\phi})
    \triangleq &
    \E_{\mathbf{G}\sim \mathrm{Unif}(\mathcal{S})}
    \E_{\bm{\Lambda}_t | \mathbf{G}}
    \|
        s_{\bm{\phi}}(\mathbf{X}_t, \bm{\Lambda}_t)
        -
        \nabla \log p_{t|0}(\bm{\Lambda}_t|\bm{\Lambda}_0)
    \|^2.
\end{align}
\normalsize

In nutshell, the proposed GSDM is summarized as three main steps. Details of GSDM can be found Algorithm \ref{alg_train} and Algorithm \ref{alg_sample} in Appendix~\ref{algorithm}.
\begin{itemize}
    \item[1.] Run forward spectral diffusion model on $(\mathbf{X}, \bm{\Lambda})$ by \eqref{spectral_fwd_sde}. Train two score networks $s_{\bm{\theta}}(\cdot), s_{\bm{\phi}}(\cdot)$ to learn score functions shown up in \eqref{spectral_reversed_time_sde}.
    \item[2.] Generate plausible $(\widehat{\mathbf{X}}_0,\widehat{\bm{\Lambda}}_0)$ from $\pi$, via  reversing the spectral diffusion SDEs from $t=1$ to $t=0$, with estimated score functions $s_{\bm{\theta}}(\widehat{\mathbf{X}}_t,\widehat{\bm{\Lambda}}_t),s_{\bm{\phi}}(\widehat{\mathbf{X}}_t,\widehat{\bm{\Lambda}}_t)$.  
    \item[3.] Generate plausible adjacency matrix via $\widehat{\mathbf{A}} = \widehat{\mathbf{U}} \widehat{\bm{\Lambda}}_0 \widehat{\mathbf{U}}^\top$, where $\widehat{\mathbf{U}}$ is uniformly sampled from the observed eigenvector matrices. 
\end{itemize}

Since the full adjacency matrices are not involved in the computation of diffusion SDEs, GSDM achieves significant acceleration in both training and sampling. 
Moreover, one can further enhance the computational efficiency, by confining the spectral diffusion to the top-$k$ largest eigenvalues of $\mathbf{A}$, where $k\triangleq [\alpha n]$. Suppose $\bm{\Lambda}^{k}$ is the truncated eigenvalue matrix, where only the top-$k$ diagonal entries of $\bm{\Lambda}$ are preserved, we can define the $\alpha$-quantile GSDM by substituting the occurrence of $\bm{\Lambda}$ in GSDM with $\bm{\Lambda}^{(k)}$. As will be seen in the following section, $\alpha$-quantile GSDM exhibits evidently faster processing speed and comparable performance to GSDM.

\subsection{Theoretical Analysis}

Here, we provide supportive theoretical evidence for the efficacy of GSDM. In Proposition \ref{Prop_spectral_SDE_adj}, we first study the low-rank structure of the spectral diffusion SDE of adjacency matrices. In Proposition \ref{prop2_error_bound}, we further prove that our proposed spectral diffusion enjoys a sharper reconstruction error bound than the standard graph diffusion model.

\begin{proposition}[Spectral Diffusion SDEs on Adjacency Matrix]\label{Prop_spectral_SDE_adj}
    Suppose $\mathbf{f}_{\Lambda}(\bm{\Lambda},t) \triangleq -\sigma_{\Lambda,t}^2/2 \bm{\Lambda}$. The spectral diffusion SDE system \eqref{spectral_fwd_sde} induces an $n$-dimensional SDE system of the adjacency matrix $\mathbf{A}$ on the full space $\R^{n\times n}$. Following previous notations, the forward diffusion SDE is given by
    \begin{align}
        \begin{cases}
            \mathrm{d} \mathbf{X}_t
            = &
            \mathbf{f}^X(\mathbf{X}_t, t) \mathrm{d}t
            +
            \sigma_{X,t} \mathrm{d} \mathbf{B}_t^X,
            \\
            \mathrm{d} \mathbf{A}_t
            = &
            -\frac{1}{2} \sigma_{\Lambda,t}^2 
            \mathbf{A}_t
            \mathrm{d}t
            +
            \sigma_{\Lambda,t} \mathrm{d} \mathbf{M}_t,
            \label{prop1_lam_sde}
        \end{cases}
    \end{align}
    where $(\mathbf{M}_t)_{t\in[0,1]}$ is a $n$-dimensional centered Gaussian process on $\mathbb{R}^{n\times n}$, with zero mean and covariance kernel $\mathcal{K}(s,t): [0,1]\times [0,1] \mapsto \mathbb{R}^{n\times n \times n\times n}$ as
    \begin{equation*}
        \mathcal{K}(s,t)_{i,j,k,l}
        =
        \min(s, t) 
        \cdot
        \sum_{h=1}^n 
        \mathbf{U}_0[i,h]
        \mathbf{U}_0[j,h]
        \mathbf{U}_0[k,h]
        \mathbf{U}_0[l,h].
    \end{equation*}
    Hence, the conditional distribution of $\mathbf{A}_t$ on $\mathbf{A}_0$ is Gaussian
    \small
    \begin{align}
        \mathbf{A}_t | \mathbf{A}_0
        \sim
        N(
            \mathbf{A}_t;
            \mathbf{A}_0 e^{-\frac{1}{2} \int_0^t \sigma_{\tau}^2 \mathrm{d} \tau},
            (1 - e^{- \int_0^t \sigma_{\tau}^2 \mathrm{d} \tau}) \mathcal{K}(1,1)
        ),
    \end{align}
    \normalsize
    which admits a closed-form probability density function.
\end{proposition}

\begin{remark}
    The proof is postponed to Appendix \ref{proof_prop1}. Proposition \ref{Prop_spectral_SDE_adj} shows that our spectral diffusion framework substantially recasts the evolution of adjacency matrix, by driving the $n^2$-dimensional SDE with $n$-dimensional noise $(\mathbf{M}_t)_{t\in \R}$. Guided by the graph spectrum, the diffusion is concentrated to the salient parts of $\mathrm{supp}(\mathbf{A})$ to prevent introducing irreducible noise to the out-of-support regions.
\end{remark}

The central question for graph generation is how to measure the quality of the synthesis data that is recovered from noise with the learned score function. The key to this question is to establish the reconstruction error bound, i.e. the expected error between the data reconstructed with ground truth score $\nabla \log p_t(\cdot)$ and the learned scores $s_{\bm{\phi}}(\cdot)$. For graph topology generation, our GSDM is proven to enjoy a sharper reconstruction bound than standard graph diffusion. Detailed proof can be found in Appendix \ref{proof_prop2}.

\begin{proposition}[Reconstruction Bounds for Adjacency Matrix Generation]\label{prop2_error_bound}
Following the previous notations, we define the estimated reversed time SDEs for standard and spectral graph diffusion models as 
\begin{align*}
    \mathrm{d}
    \widehat{\mathbf{A}}_t^{\mathrm{full}}
    = &
    \left(
        - \frac{1}{2} \sigma_t^2
        \widehat{\mathbf{A}}_t^{\mathrm{full}}
        -
        \sigma_{t}^2
        s_{\bm{\phi}}(\widehat{\mathbf{A}}_t^{\mathrm{full}},t)
    \right)
    \mathrm{d}\bar{t}
    +
    \sigma_{t} \mathrm{d} \bar{\mathbf{B}}_t^{A},
    \\
    \mathrm{d}
    \widehat{\mathbf{A}}_t^{\mathrm{spec}}
    = &
    \left(
        - \frac{1}{2} \sigma_t^2
        \widehat{\mathbf{A}}_t^{\mathrm{spec}}
        -
        \sigma_{t}^2
        s_{\bm{\varphi}}(\widehat{\mathbf{A}}_t^{\mathrm{spec}},t)
    \right)
    \mathrm{d}\bar{t}
    +
    \sigma_{t} \mathrm{d} \bar{\mathbf{M}}_t,
\end{align*}
where both generation methods share the same drift term and noise schedule $(\sigma_t)_{t\in[0,1]}$. We further assume that $\|s_{\bm{\phi}}(\cdot,t)\|_{\mathrm{lip}} = \mathcal{O}(\E_{|\mathbf{A}_0}\|\nabla_{\mathbf{A}} \log p_{t|0}(\widehat{\mathbf{A}}_t^{\mathrm{full}})\|_{\mathrm{lip}})$ and $\|s_{\bm{\varphi}}(\cdot,t)\|_{\mathrm{lip}} = \mathcal{O}(\E_{|\mathbf{A}_0}\|\nabla_{\mathbf{A}} \log p_{t|0}(\widehat{\mathbf{A}}_t^{\mathrm{spec}})\|_{\mathrm{lip}})$ holds almost surly. By reversing both SDEs from $t=1$ to $t=0$, we obtain $\widehat{\mathbf{A}}_t^{\mathrm{full}}$ and $\widehat{\mathbf{A}}_t^{\mathrm{spec}}$, which are two reconstructions of the authentic $\mathbf{A}_0$, with expected reconstruction errors bounded by
\begin{align*}
    \notag
    & \E\|
    \mathbf{A}_0
    -
    \widehat{\mathbf{A}}_0^{\mathrm{full}}
    \|^2
    \\
    & \leqslant
    M
    \mathcal{E}(\bm{\phi})
    \cdot
    \bigg(
        1
        +
        n^2 K
        \int_0^1
            \Sigma_t^{-2}
            \exp
            \left(
            n^2 K
            \int_t^1
                \Sigma_s^{-2}
            \mathrm{d}s
            \right)
        \mathrm{d}t
    \bigg),
    \\
    \notag
    & \E\|
    \mathbf{A}_0
    -
    \widehat{\mathbf{A}}_0^{\mathrm{spec}}
    \|^2
    \\
    & \leqslant
    M
    \mathcal{E}(\bm{\varphi})
    \cdot
    \left(
        1
        +
        nK
        \int_0^1
            \Sigma_t^{-2}
            \exp
            \left(
            nK
            \int_t^1
                \Sigma_s^{-2}
            \mathrm{d}s
            \right)
        \mathrm{d}t
    \right),
\end{align*}
where $M \triangleq C^2 \|\sigma_{\cdot}\|_{\infty}^4$; $K\triangleq 2ML/\E \|\mathbf{A}_0\|_{2,2}$; $C, L$ are absolute constants; 
$
\Sigma_t^2
\triangleq
1 - e^{
        -
        \int_0^t
            \sigma_s^2
        \mathrm{d}s
        }   
$; $\mathcal{E}(\cdot)$ is the expected score matching objective defined in \eqref{sm_objective}.

\end{proposition}

\begin{remark}
While the error bound of $\widehat{\mathbf{A}}_0^{\mathrm{full}}$ is at the order of $\mathcal{O}(n^2 \exp(n^2))$,  $\widehat{\mathbf{A}}_0^{\mathrm{spec}}$ exhibits a much sharper bound of order $\mathcal{O}(n \exp(n))$. For large-scale graphs with a large number of nodes $n$, our proposed GSDM enjoys a substantially better performance guarantee, which coincides with our numerical results. Moreover, under additional conditions, the error bounds can be significantly improved for well-trained score networks $\bm{\phi}^*, \bm{\varphi}^*$, by showing that $\mathcal{E}(\bm{\phi}^*) \ll \mathcal{E}(\bm{\phi})$ as in Proposition \ref{Prop_sm_converge}. The detailed proof can be found in Appendix~\ref{proof_prop3}. 
\end{remark}

\begin{definition}[$\beta$-smooth]
    $f: \mathbb{R}^m \mapsto \mathbb{R}^n$ is called $\beta$-smooth if and only if
    \begin{equation*}
        \|f(x_1) - f(x_2) - \nabla f(x_2)^\top (x_1 - x_2) \|
        \leqslant
        \frac{\beta}{2}
        \|x_1 - x_2\|^2
    \end{equation*}
    holds for $\forall\ x_1, x_2 \in \mathbb{R}^m$.
\end{definition}

\begin{proposition}[Convergence of Score Matching Objective Minimization]\label{Prop_sm_converge}
    Recall that the score matching objective of a generic sample $\mathbf{Z} \in \mathbb{R}^d$ ($\widehat{\mathbf{A}}^{\mathrm{full}}$ or $\widehat{\mathbf{A}}^{\mathrm{spec}}$) is defined as
    \begin{align}
        \mathcal{E}(\bm{\theta};\mathbf{Z})
        \triangleq &
        \mathbb{E}_{\mathbf{Z}_t|\mathbf{Z}} 
        \|
            s_{\bm{\theta}}(\mathbf{Z}_t,t) 
            - 
            \nabla_{\mathbf{Z}} \log p_t(\mathbf{Z}_t)
        \|^2,
    \end{align}
    and the expected and empirical score matching error are defined as
    \begin{align}
        \mathcal{E}(\bm{\theta})
        \triangleq &
        \mathbb{E}_{\mathbf{Z}\sim \mathcal{D}} 
        \mathcal{E}(\bm{\theta};\mathbf{Z})
        \\
        \widehat{\mathcal{E}}(\bm{\theta}; \mathcal{S}_N)
        \triangleq &
        \mathbb{E}_{\mathbf{Z}\sim \mathcal{S}_N} 
        \mathcal{E}(\bm{\theta};\mathbf{Z})
    \end{align}
    where $\mathcal{D}$ is the population distribution of $\mathbf{Z}$, and $\mathcal{S}\triangleq \{\mathbf{Z}^i\}_{i=1}^N$ is the uniform distribution over an i.i.d sampled training dataset of size $N$.
    Suppose $\mathcal{E}(\bm{\theta})$ is minimized via running standard Stochastic Gradient Descent (SGD) on training data, i.e. at the $k$-th iteration, $\bm{\theta}_k$ is updated on a mini-batch of size $b$
    \begin{align}
        \bm{\theta}_{k+1} 
        \triangleq &
        \bm{\theta}_k - \eta \nabla_{\bm{\theta}} \widehat{\mathcal{E}}(\bm{\theta}; \mathcal{S}_b).
    \end{align}
    Assume that almost surely (w.r.t $\mathbf{Z}$), $\mathcal{E}(\cdot;\mathbf{Z})$ is $\beta$-smooth, and the tangent kernel $\mathbf{K}_{\bm{\theta}}(\mathcal{S})\in \mathbb{R}^{Nd\times Nd}$ of $s_{\bm{\theta}}(\cdot)$ satisfies
    \begin{align*}
        & \lambda_{\mathrm{min}}
        \left(
            \mathbf{K}_{\bm{\theta}}(\mathcal{S})
        \right)
        \geqslant
        \lambda > 0
        , \
        \bm{\theta} \in B(\bm{\theta}_0, R),
        \\
        & \mathbf{K}_{\bm{\theta}}(\mathcal{S})[Ni+1:N(i+1), Nj+1:N(j+1)]
        \\
        & \triangleq 
        \nabla_{\bm{\theta}}s_{\bm{\theta}}(\mathbf{Z}^i)^\top
        \nabla_{\bm{\theta}}s_{\bm{\theta}}(\mathbf{Z}^j),
    \end{align*}
    with $R = 2N\sqrt{2\beta \mathcal{E}(\bm{\theta}_0)}/ (\mu \delta)$, $\delta > 0$. Then, with probability $1-\delta$ over the choice of mini-batch $\mathcal{S}_b$, SGD with a learning rate $\eta \leqslant \frac{\lambda / N}{N\beta (N^2\beta + \lambda(b-1)/N)}$ converges to a global solution in the ball $B(\bm{\theta}_0,R)$ with exponential convergence rate
    \begin{align}
        \mathcal{E}(\bm{\theta}_k)
        \leqslant
        \left(
            1 
            -
            \frac{\lambda b \eta}{N}
        \right)^k
        \mathcal{E}(\bm{\theta}_0).
    \end{align}
\end{proposition}

\section{Experiments}

\begin{table*}[ht]
\caption{ \textbf{Generation results on the generic graph datasets.} We report the MMD distances between the test datasets and generated graphs. The best results are highlighted in bold (the smaller the better). 
    Hyphen (-) denotes out-of-resources that take more than 10 days or are not applicable due to memory issues. 
    }
    \label{tab:generic}
\begin{threeparttable}
    
    \centering
    \resizebox{\textwidth}{!}{
    \renewcommand{\arraystretch}{1.1}
    \renewcommand{\tabcolsep}{8pt}
    \begin{tabular}{l l c c c c c c c c c c c c c c c c}
    \toprule
        & & 
        \multicolumn{4}{c}{{Community-small}} &
        \multicolumn{4}{c}{{Enzymes}} &
        \multicolumn{4}{c}{{Grid}}\\
    \cmidrule(l{2pt}r{2pt}){3-6}    
    \cmidrule(l{2pt}r{2pt}){7-10}
    \cmidrule(l{2pt}r{2pt}){11-14}
    \cmidrule(l{2pt}r{2pt}){15-18}
        & &
        \multicolumn{4}{c}{Synthetic, $12\leq|V|\leq20$} &
        \multicolumn{4}{c}{Real, $10\leq|V|\leq125$} &
        \multicolumn{4}{c}{Synthetic, $100\leq|V|\leq400$} \\
    \cmidrule(l{2pt}r{2pt}){3-6}
    \cmidrule(l{2pt}r{2pt}){7-10}
    \cmidrule(l{2pt}r{2pt}){11-14}
    \cmidrule(l{2pt}r{2pt}){15-18}
        &  & Deg.$\downarrow$ & Clus.$\downarrow$ & Orbit$\downarrow$ & Avg.$\downarrow$ & Deg.$\downarrow$ & Clus.$\downarrow$ & Orbit$\downarrow$ & Avg.$\downarrow$ & Deg.$\downarrow$ & Clus.$\downarrow$ & Orbit$\downarrow$ & Avg.$\downarrow$ \\
    \midrule
        \multirow{4}{*}{Autoreg.}
        & DeepGMG \cite{li2018learning} &  0.220 & 0.950 & 0.400 & 0.523 & - & - & - & - & - & - & - & - \\
        & GraphRNN \cite{you2018graphrnn} &  0.080 & 0.120 & 0.040 & 0.080 &  {0.017} &  \textbf{0.043} & {0.021} & {0.043} \\
        & GraphAF \cite{shi2020graphaf}  & 0.18 & 0.20 & 0.02 & 0.133 & 1.669 & 1.283 & 0.266 & 1.073 & - & - & - & - \\
        & GraphDF \cite{luo2021graphdf} & 0.06 & 0.12 & 0.03 & 0.070 & 1.503 & 1.061 & 0.202 & 0.922 & - & - & - & - \\
    \midrule
        \multirow{5.5}{*}{One-shot}
        & GraphVAE \cite{simonovsky2018graphvae} & 0.350 & 0.980 & 0.540 & 0.623 & 1.369 & 0.629 & 0.191 & 0.730 & 1.619 & \textbf{0.0} & 0.919 & 0.846 \\
        & GNF \cite{liu2019graph}  & 0.200 & 0.200 & 0.110 & 0.170 & - & - & - & - & - & - & - & - \\
        & EDP-GNN \cite{niu2020permutation} & 0.053 & 0.144 & 0.026 & 0.074 & 0.023 & 0.268 & 0.082 & 0.124 & 0.455 & 0.238 & 0.328 & 0.340 \\

        & {GDSS}\footnote[1]{It} \cite{jo2022score} & {0.045} & {0.086} & {0.007} & {0.046} & 0.026 & {0.102} & {0.009} & {0.046} & 0.111 & 0.005 & 0.070 & 0.062 \\
    \cmidrule(l{4pt}r{2pt}){2-18}
        & \textbf{Ours} & \textbf{0.011}&	\textbf{0.015}&	\textbf{0.001}&	\textbf{0.009}& \textbf{0.013}& 0.088 &\textbf{0.01}&	\textbf{0.037}& \textbf{0.002}&	\textbf{0.0}&	\textbf{0.0}&	\textbf{0.0007} \\
    \bottomrule
    \end{tabular}}
\begin{tablenotes}
\item[1] The average results of the Enzymes dataset reported in the GDSS original paper is 0.032. However, the best result we can obtain using the 

author's released code and checkpoint with careful fine-tuning is 0.046.
\end{tablenotes}
\end{threeparttable}
\end{table*}

In this section, we evaluate our proposed GDSM with state-of-the-art graph generative baselines on two types of graph generation tasks: generic graph generation and molecule generation, over several benchmark datasets. We also conduct extensive ablation studies and visualizations to further illustrate the effectiveness and efficiency of GDSM.

\begin{figure}[htbp]
    \centering    
    \includegraphics[width=0.85\columnwidth]{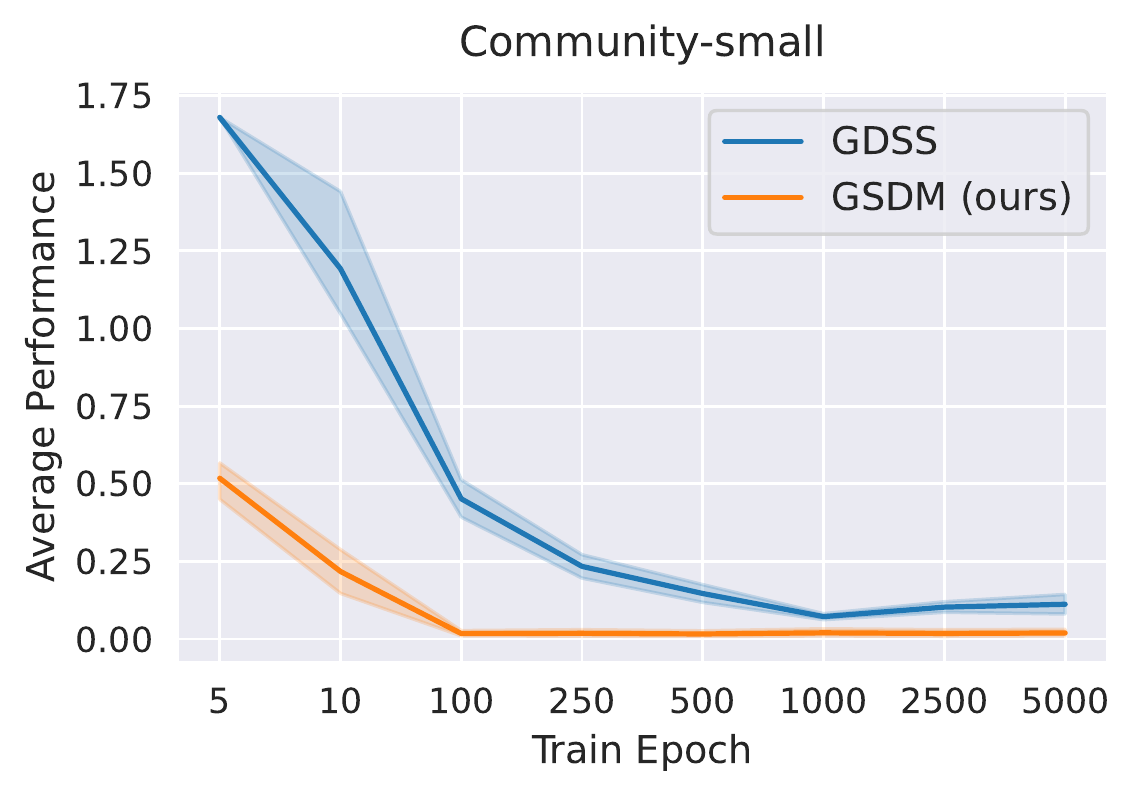}
    \caption{GSDM enjoys a significantly faster convergence rate than GDSS. On the community-small dataset, our GSDM reaches the SOTA performance within 100 training epochs.}
    \label{convergence}
\end{figure}

\subsection{Generic Graph Generation}



\noindent \textbf{Datasets} We test our model on three generic datasets with different scales, and we use $N$ to denote the number of nodes: (1) Community-small $(12\leq N\leq20)$: contains 100 small community graphs. (2) Enzymes $(10\leq N\leq125)$: contains 578 protein graphs which represent the protein tertiary structures of the enzymes from the BRENDA database. (3) Grid $(100 \leq N \leq 400)$: contains 100 standard 2D grid graphs. To fairly compare our model with baselines, we adopt the experimental and evaluation setting from \cite{you2018graphrnn} with the same train/test split. 

\noindent \textbf{Baselines} We compare our model with well-known or state-of-the-art graph generation methods, which can be categorized into auto-regressive models and one-shot models. The \textbf{auto-regressive model} refers to sequential generation, which constructs a graph via a set of consecutive steps, usually done nodes by nodes and edges by edges \cite{guo2020systematic}. Under this category, we include DeepGMG \cite{li2018learning}, GraphRNN \cite{you2018graphrnn}, GraphAF \cite{shi2020graphaf}, and GraphDF \cite{luo2021graphdf}. The \textbf{one-shot model} refers to building a probabilistic graph model that can generate all nodes and edges in one shot \cite{guo2020systematic}. Under this category, we include GraphVAE \cite{simonovsky2018graphvae}, Graph Normalizing Flow(GNF) \cite{liu2019graph}, EDP-GNN \cite{niu2020permutation}, and GDSS \cite{jo2022score}.



\noindent \textbf{Metrics} We adopt the maximum mean discrepancy (MMD) to compare the distributions of graph statistics, such as the degree, the clustering coefficient, and the number of occurrences of orbits with 4 nodes~\cite{you2018graphrnn,jo2022score} between the same number of generated and test graphs.



\begin{table*}[ht]
    \caption{\small \textbf{Generation results on the QM9 and ZINC250k datasets.} Results are the means of three different runs, and the best results are highlighted in bold. Values denoted by * are taken from the respective original papers. Other results are obtained by running open-source codes. Val. w/o corr. denotes the Validity w/o correction metric, and values that do not exceed 50\% are underlined.}
    \centering
    \resizebox{\textwidth}{!}{
    \renewcommand{\arraystretch}{1.1}
    \renewcommand{\tabcolsep}{8pt}
    \begin{tabular}{llcccccccccc}
    \toprule
    & & \multicolumn{5}{c}{QM9} & \multicolumn{5}{c}{ZINC250k} \\
    \cmidrule(l{2pt}r{2pt}){3-7}
    \cmidrule(l{2pt}r{2pt}){8-12}
        & Method & Validity (\%)$\uparrow$ & Val. w/o corr. (\%)$\uparrow$ & NSPDK$\downarrow$ & FCD$\downarrow$ & Time (s)$\downarrow$ &Validity (\%)$\uparrow$& Val. w/o corr. (\%)$\uparrow$ & NSPDK$\downarrow$ & FCD$\downarrow$ & Time (s)$\downarrow$ \\
    \midrule
        \multirow{4}{*}{Autoreg.}
        & GraphAF~\cite{shi2020graphaf} & 100 &67* & 0.020 & 5.268 & 2.28$e^{3}$ &100& 68* & 0.044 & 16.289 & 5.72$e^{3}$ \\
        & GraphAF+FC & 100&74.43 & 0.021 & 5.625 & 2.32$e^{3}$ & 100 &68.47 & 0.044 & 16.023 & 5.91$e^{3}$ \\
        & GraphDF~\cite{luo2021graphdf} & 100 &82.67* & 0.063 & 10.816 & 5.08$e^{4}$ &100& 89.03* & 0.176 & 34.202 & 5.87$e^{4}$ \\
        & GraphDF+FC &100 &93.88 & 0.064 & 10.928 & 4.72$e^{4}$ &100& 90.61 & 0.177 & 33.546 & 5.79$e^{4}$ \\
    \midrule
        \multirow{5.5}{*}{One-shot}
        & MoFlow~\cite{zang2020moflow} &100 &91.36 & 0.017 & 4.467 & \textbf{4.58} &100 & 63.11 & 0.046 & 20.931 & \textbf{25.9} \\
        & EDP-GNN~\cite{niu2020permutation} &100 &\underline{47.52} & 0.005 & {2.680} & 4.13$e^{3}$ &100& 82.97 & 0.049 & 16.737 & 8.41$e^{3}$ \\
        & GraphEBM~\cite{liu2021graphebm} &100 &\underline{8.22} & 0.030 & 6.143 & 35.33 &100 &\underline{5.29} & 0.212 & 35.471 & 53.72 \\

        & {GDSS}~\cite{jo2022score} &100 &95.72* & \textbf{0.003}* & 2.900* & 1.06$e^{2}$ & 100& \textbf{97.01}* & 0.019* & 14.656* & 2.11$e^{3}$ \\
    \cmidrule(l{2pt}r{2pt}){2-12}
        & \textbf{Ours} &100 & \textbf{99.9} & \textbf{0.003} & \textbf{2.650} & 18.02 & 100 &92.70 & \textbf{0.017} & \textbf{12.956} &  45.91\\
    \bottomrule
    \end{tabular}}
    \label{tab:mol}
\end{table*}



\noindent \textbf{Results} We show the results in Table \ref{tab:generic}. The results show that our proposed GSDM significantly outperforms both the auto-regressive baselines and one-shot baseline methods. Specifically, GDSS is the SOTA graph diffusion model which performs the diffusion in the whole space of the graph data. Our method substantially outperforms GDSS in terms of both average performance and convergence rate (Figure \ref{convergence}), which demonstrates the advantages of performing SDEs in the spectrum of the graph compared to the whole space. 


\subsection{Molecules Generation}

Besides generic graph generation, our model can also generate organic molecules through our proposed reverse diffusion process. We test our model with two well-known molecule datasets: QM9 \cite{ramakrishnan2014quantum} and ZINC250K \cite{irwin2012zinc}. Following previous works \cite{jo2022score,luo2021graphdf}, the molecules are kekulized by the RDKit library \cite{landrum2016rdkit} with hydrogen atoms removed. We evaluate the quality of 10,000 generated molecules with \textbf{Frechet ChemNet Distance (FCD)} \cite{preuer2018frechet}, \textbf{Neighborhood subgraph pairwise distance kernel (NSPDK) MMD} \cite{costa2010fast}, \textbf{validity w/o correction}, and the \textbf{generation time}. FCD computes the distance between the testing and the generated molecules using the activations of the
penultimate layer of the ChemNet. (NSPDK) MMD computes the MMD between the generated and the testing set which takes into account both
the node and edge features for evaluation. Generally speaking, FCD measures the generation quality in the view of molecules in the chemical space, while NSPDK MMD evaluates the generation quality from the graph structure perspective. Besides, following \cite{jo2022score}, we also include the \textbf{validity w/o correction} as another metric to explicitly evaluate the quality of molecule generation prior to the correction procedure. It computes the fraction of the number of valid molecules without valency correction or edge resampling over the total number of generated molecules. In contrast, \textbf{validity} measures the fraction of the valid molecules after the correction phase. \textbf{Generation time} measures the time for generating 10,000 molecules in the form of RDKit. It is a salient measure of the practicability of molecule generation, especially for macromolecule generation. 

\begin{table*}[ht]
    \caption{ \textbf{Ablation study on the $\alpha$-quantile eigenvalues.} The metrics used are the same as in Table 1. The results that \textbf{surpass} the baseline methods in Table 1 are highlighted in bold (the smaller the better). Avg.\% denotes the percentage of average scores achieved compared to using the whole eigenvalues and eigenvectors set.}
    \label{tab:eigen}
    \centering
    \resizebox{\textwidth}{!}{
    \renewcommand{\arraystretch}{1.1}
    \renewcommand{\tabcolsep}{8pt}
    \begin{tabular}{l l c c c c c c c c c c c c c c c c}
    \toprule
        & 
        \multicolumn{5}{c}{{Community-small}} &
        \multicolumn{5}{c}{{Enzymes}} &
        \multicolumn{5}{c}{{Grid}}\\
    \cmidrule(l{2pt}r{2pt}){2-6}    
    \cmidrule(l{2pt}r{2pt}){7-11}
    \cmidrule(l{2pt}r{2pt}){12-16}
        &
        \multicolumn{5}{c}{Synthetic, $12\leq|V|\leq20$} &
        \multicolumn{5}{c}{Real, $10\leq|V|\leq125$} &
        \multicolumn{5}{c}{Synthetic, $100\leq|V|\leq400$} \\
    \cmidrule(l{2pt}r{2pt}){2-6}
    \cmidrule(l{2pt}r{2pt}){7-11}
    \cmidrule(l{2pt}r{2pt}){12-16}
        $\alpha$ &  Deg.$\downarrow$ & Clus.$\downarrow$ & Orbit$\downarrow$ & Avg.$\downarrow$ & Avg.\%$\uparrow$ &  Deg.$\downarrow$ & Clus.$\downarrow$ & Orbit$\downarrow$ & Avg.$\downarrow$ & Avg.\%$\uparrow$ & Deg.$\downarrow$ & Clus.$\downarrow$ & Orbit$\downarrow$ & Avg.$\downarrow$  & Avg.\%$\uparrow$ \\
    \midrule

        10\% &1.010	&1.105	&0.227	&0.781&	0.012 & 0.446	&0.608&	0.059 & 0.371	&0.029& 1.996&	\textbf{0}	&1.013 & 1.003&	0.001
        \\
        20\% &  0.450	&1.102&	0.102&	0.551	&0.016 &0.056	&0.220&	0.012&	0.096&0.395 & 1.996&	\textbf{0}	&1.013 & 1.003	& 0.001\\
        30\%  &0.085	&0.421&	0.030&	0.179&	0.050 & \textbf{0.011}&	\textbf{0.093}&	\textbf{0.012}	&\textbf{0.039}&	0.982& \textbf{0.107}&	\textbf{0}& \textbf{0.044}&	\textbf{0.050}&	0.012\\
        40\%  & \textbf{0.039}&	\textbf{0.056}&	\textbf{0.006}	&\textbf{0.034}	&0.268& \textbf{0.010}	&\textbf{0.093}&	\textbf{0.011}	&\textbf{0.038}&	1.000 & \textbf{0.018}&	\textbf{0}	& \textbf{0.001} &\textbf{0.006}&	0.095\\
        50\% & \textbf{0.022}	&\textbf{0.024}	&\textbf{0.002}&	\textbf{0.016}&	0.563 & \textbf{0.010}	&\textbf{0.093}&	\textbf{0.011}&	\textbf{0.038}&1.000& \textbf{0.007}&	\textbf{0}	&\textbf{0.001}&	\textbf{0.002}	&0.225\\
        \midrule
        60\%  & \textbf{0.014}&	\textbf{0.019}	&\textbf{0.001}&	\textbf{0.011}&	0.794& \textbf{0.010}&	\textbf{0.093}&	\textbf{0.011}	&\textbf{0.038} &1.000&\textbf{ 0.002}&	\textbf{0}&	\textbf{0}	&\textbf{0.001}	&1.000\\
        70\%  & \textbf{0.012}&	\textbf{0.019}&	\textbf{0.001}&	\textbf{0.011}	&0.844& \textbf{0.010}	&\textbf{0.093}	&\textbf{0.011}&	\textbf{0.038}&1.000&\textbf{0.002}	&\textbf{0}&	\textbf{0}	&\textbf{0.001}	&1.000\\

        80\%  &  \textbf{0.011}&	\textbf{0.016}&	\textbf{0.001}	&\textbf{0.009}	&0.964& \textbf{0.010}&	\textbf{0.093}&	\textbf{0.011}&	\textbf{0.038}&1.000& \textbf{0.002}&\textbf{0}&	\textbf{0}&	\textbf{0.001}&	1.000\\
        90\%  &  \textbf{0.011}	&\textbf{0.015}	&\textbf{0.001}&	\textbf{0.009}	&1.000& \textbf{0.010}&	\textbf{0.09}3&	\textbf{0.011}&	\textbf{0.038}&1.000& \textbf{0.002}&	\textbf{0}&	\textbf{0}&	\textbf{0.001}&	1.000\\
        100\%  &  \textbf{0.011}&\textbf{	0.015}&\textbf{	0.001}	&\textbf{0.009}	&1.000& \textbf{0.010}&\textbf{0.093}	&\textbf{0.011}	&\textbf{0.038}&1.000&\textbf{0.002}&\textbf{0}&	\textbf{0}	&\textbf{0.001}	&1.000\\
    \bottomrule
    \end{tabular}}
\end{table*}

\noindent \textbf{Baselines} We compare our model with the state-of-the-art molecule generation models. The baselines include SOTA auto-regressive models: GraphAF \cite{shi2020graphaf} is a flow-based model, and GraphDF \cite{luo2021graphdf} is a flow-based model using discrete latent variables. Following GDSS \cite{jo2022score}, we modify the architecture of GraphAF and GraphDF to consider formal charges in the molecule generation, denoted as GraphAF+FC and GraphDF+FC, for fair comparisons. For the one-shot model, we include MoFlow \cite{zang2020moflow}, which is a flow-based model; EDP-GNN \cite{niu2020permutation} and GDSS \cite{niu2020permutation} which are both diffusion models.


\noindent \textbf{Results} We show the results in Table \ref{tab:mol}. Evidentally, GSDM achieves the highest performance under most of the metrics.  The highest scores in NSPDK and FCD show that GSDM is able to generate molecules that have close data distributions to the real molecules in both the chemical space and graph space. Especially, our model outperforms GDSS, in most of the metrics, verifying that our proposed spectral diffusion is not only suitable for generic graph generation but also advisable for molecule designs. Apart from the visualization of generated graphs in Figure \ref{fig:gen}, visualizations of generated molecules are included in Appendix~\ref{molecule_visualization}, showing that the proposed GSDM is capable of generating generic graphs as well as molecules of high quality.

\noindent \textbf{Time Complexity} One of the main advantages of our GSDM compared to other diffusion models (e.g. EDP-GNN and GDSS) is the efficiency in generating molecules. We show the time spent (in seconds) for generating 10,000 molecules in Table \ref{tab:mol}, demonstrating that our GSDM takes a significantly lower time in the inference process compared to EDP-GNN and GDSS. Especially, compared to GDSS which requires $1.06e^3$ and $2.11e^3$ seconds to generate 10,000 molecule graphs according to QM9 and Zinc250k, our model only takes 18.02 and 45.91 seconds, which are $\textbf{58}\times$ and $\textbf{46}\times$ faster respectively. This phenomenon verifies our methodology in designing the spectral diffusion that not only the spectral diffusion is more effective than the whole space diffusion, but also more efficient. Such an improvement is crucial for numerous applications such as drug design and material analysis.

\subsection{Ablation Studies}

To evaluate the effectiveness of GSDM, we conduct extensive ablation experiments and display elaborate results and discussions in this section.

\begin{figure}
    \centering
    \includegraphics[width=0.85\columnwidth]{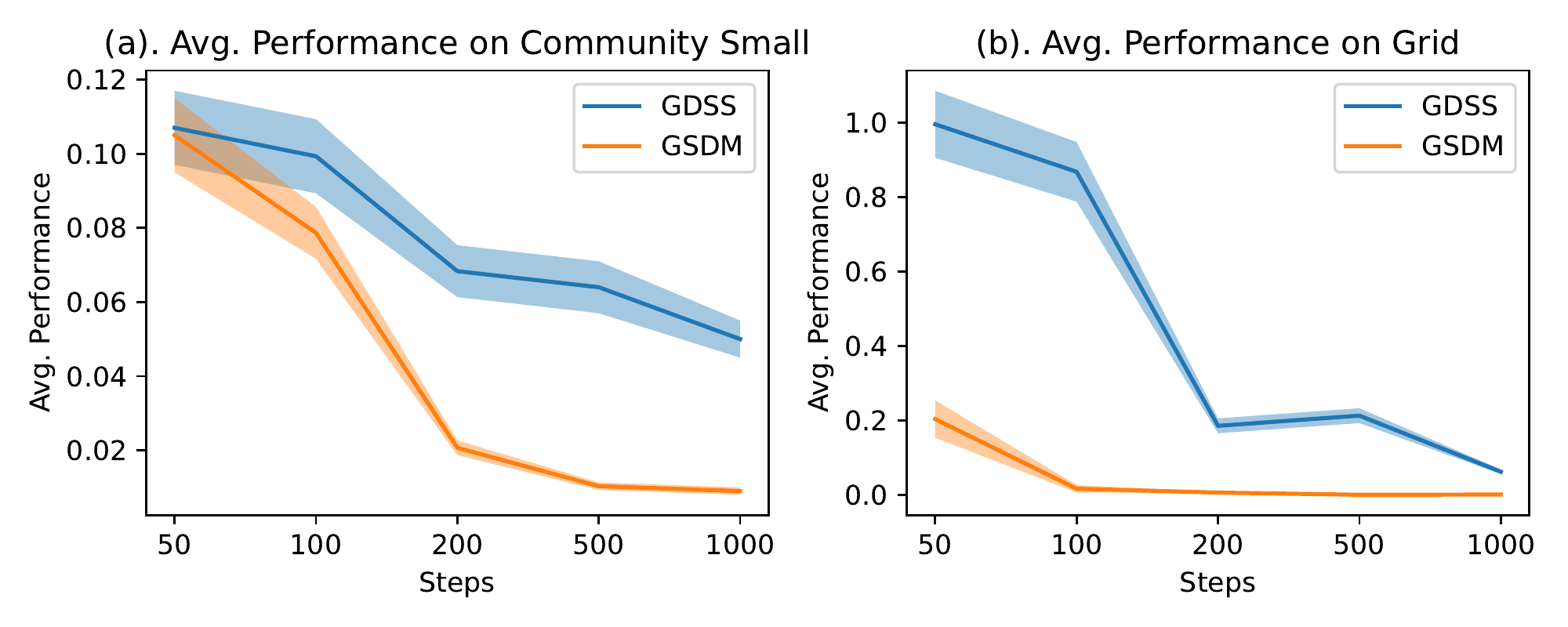}
    \caption{Ablation studies on different diffusion step numbers.}
    \label{fig:steps}
\end{figure}

\noindent \textbf{Ablation studies on the number of diffusion steps.} To analyse the robustness of our proposed GSDM during the reconstruction procedure (i.e. reverse diffusion process), we compare the performance of our model with GDSS on varying numbers of diffusion steps using the Community small and Grid dataset, which represents small and large generic graphs respectively. The default steps for our model and GDSS during training and evaluation are 1,000 steps. Ideally, for diffusion models, with the reduction in step numbers, the performance of the models decreases, while the model of higher robustness still exhibits a higher performance under different numbers of diffusion steps. The results are shown in Figure \ref{fig:steps}. We can observe that GSDM outperforms GDSS in both datasets under different diffusion steps. Specifically, in the Grid dataset, GSDM consistently exhibits significantly higher performance than GDSS. In the community small dataset, although GSDM and GDSS achieve similar performance at 50 steps, GSDM becomes much more accurate when the number of steps is higher than 100.

\begin{figure}
    \centering
    \includegraphics[width=0.85\columnwidth]{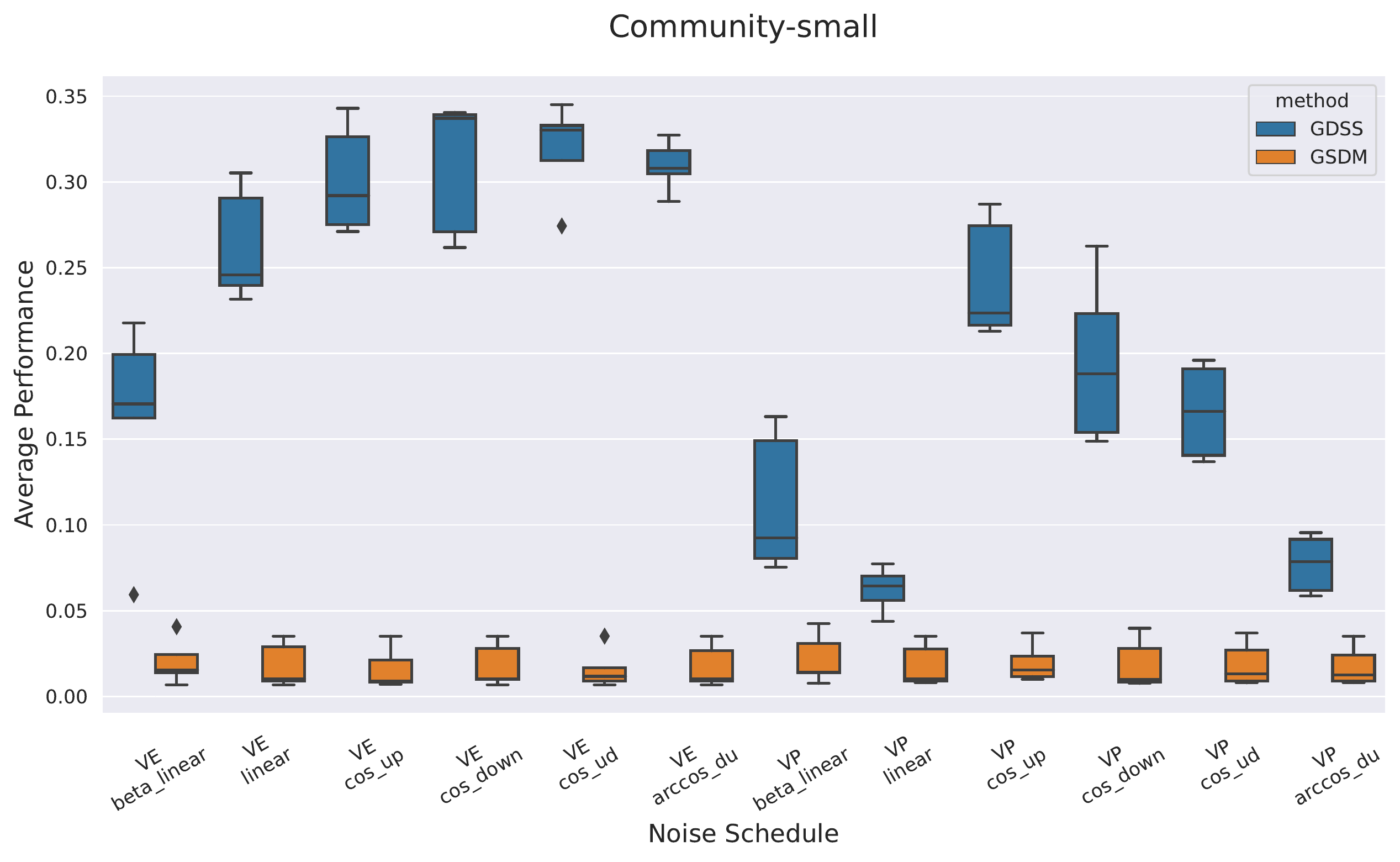}
    \caption{Ablation studies on diffusion schedules. The y-axis denotes the average performance score ($\downarrow$) and each column denotes a diffusion schedule configuration. Each box plot summarizes the results of 5 random trials.} 
    \label{diff_schedule_community_small}
\end{figure}

\noindent \textbf{Ablation studies on the type of diffusion schedules.} As mentioned in \cite{song2020score,improved_ddpm} and \cite{variational_DM}, the performance of diffusion models highly relies on the diffusion schedule, i.e.  the scheduled noise insertion process or discretization of certain types of SDEs, which lies at the heart of both training and inference phases. According to the taxonomy proposed in \cite{song2020score} and \cite{song_ncsn}, mainstream artificial diffusion schedules are categorized 
to be either Variance Exploding (VE) or Variance Preserving (VP). To improve model robustness against the choice of diffusion schedules, \cite{variational_DM} and \cite{improved_ddpm} designed diffusion models with optimizable diffusion schedules. In this subsection, we conduct systematic experiments on several datasets to test the robustness of our proposed GSDM against different diffusion schedules. As illustrated in Figure \ref{diff_schedule_community_small}, the state-of-the-art performance of our GSDM is immune to the choice of diffusion schedule and it frees us from tedious hyperparameter searching. When comparing with GDSS, our GSDM exhibits evidently better average performance with much smaller variance across various configurations. More results and experiment details are included in Appendix \ref{diffusion_schedule_ablation}.

\begin{figure}
     \centering
     \begin{subfigure}[b]{0.32\linewidth}
         \centering
         \includegraphics[width=\textwidth]{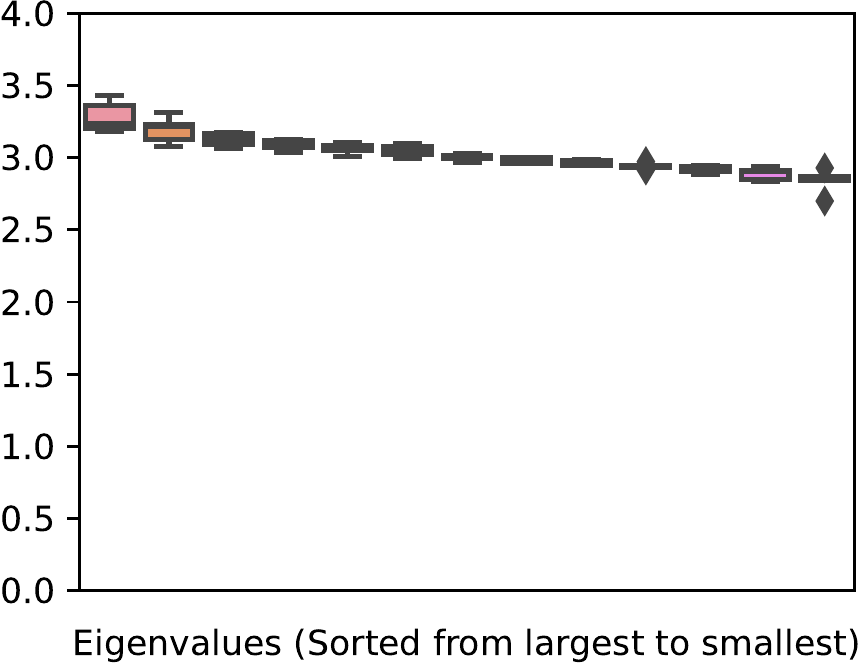}
         \caption{Synthetic-1}
     \end{subfigure}
     \hfill
     \begin{subfigure}[b]{0.32\linewidth}
         \centering
         \includegraphics[width=\textwidth]{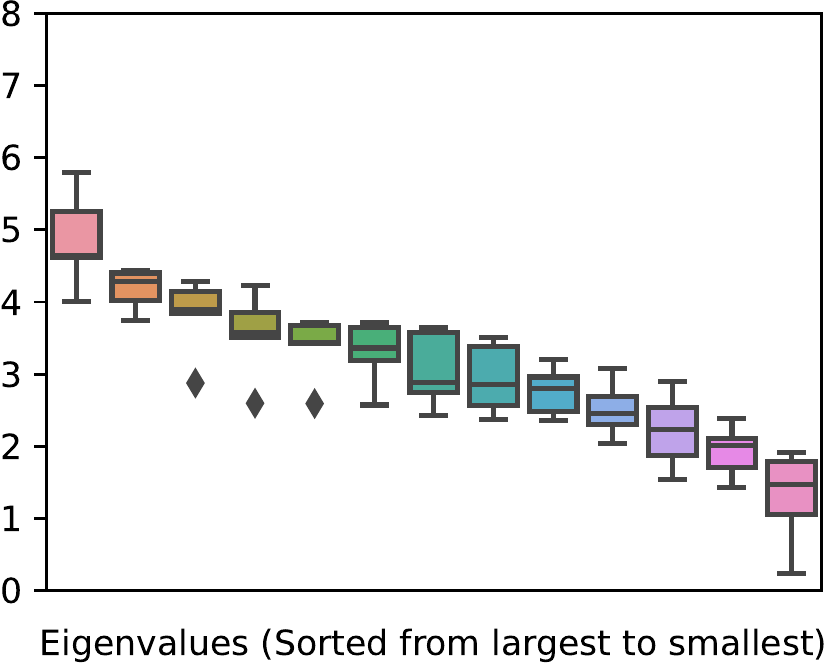}
         \caption{Synthetic-2}
     \end{subfigure}
     \hfill
     \begin{subfigure}[b]{0.32\linewidth}
         \centering
         \includegraphics[width=\textwidth]{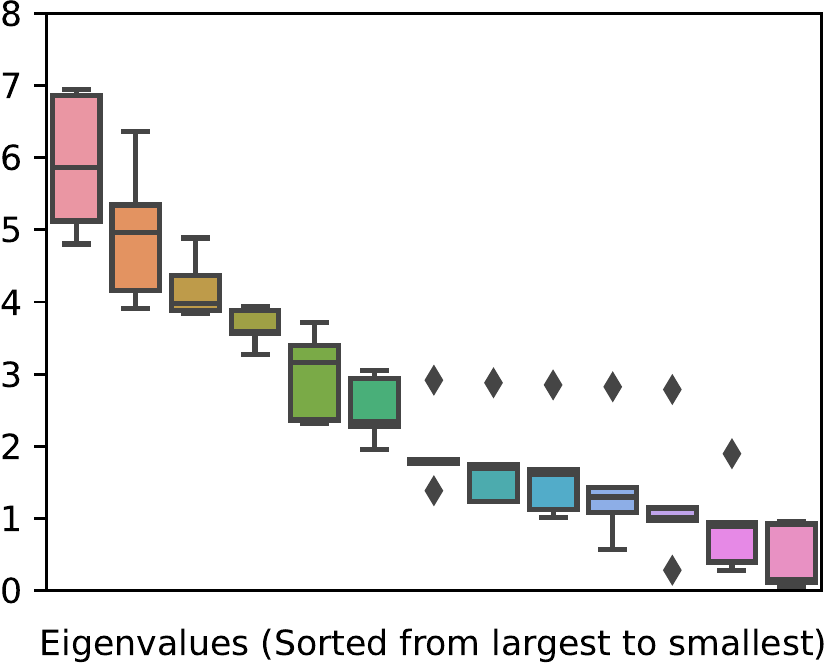}
         \caption{Synthetic-3}
     \end{subfigure}
        \caption{The eigenvalue distribution of 3 synthetic datasets.}
        \label{fig:three synthetic}
\end{figure}

\begin{figure}
    \centering
    \includegraphics[width=0.85\columnwidth]{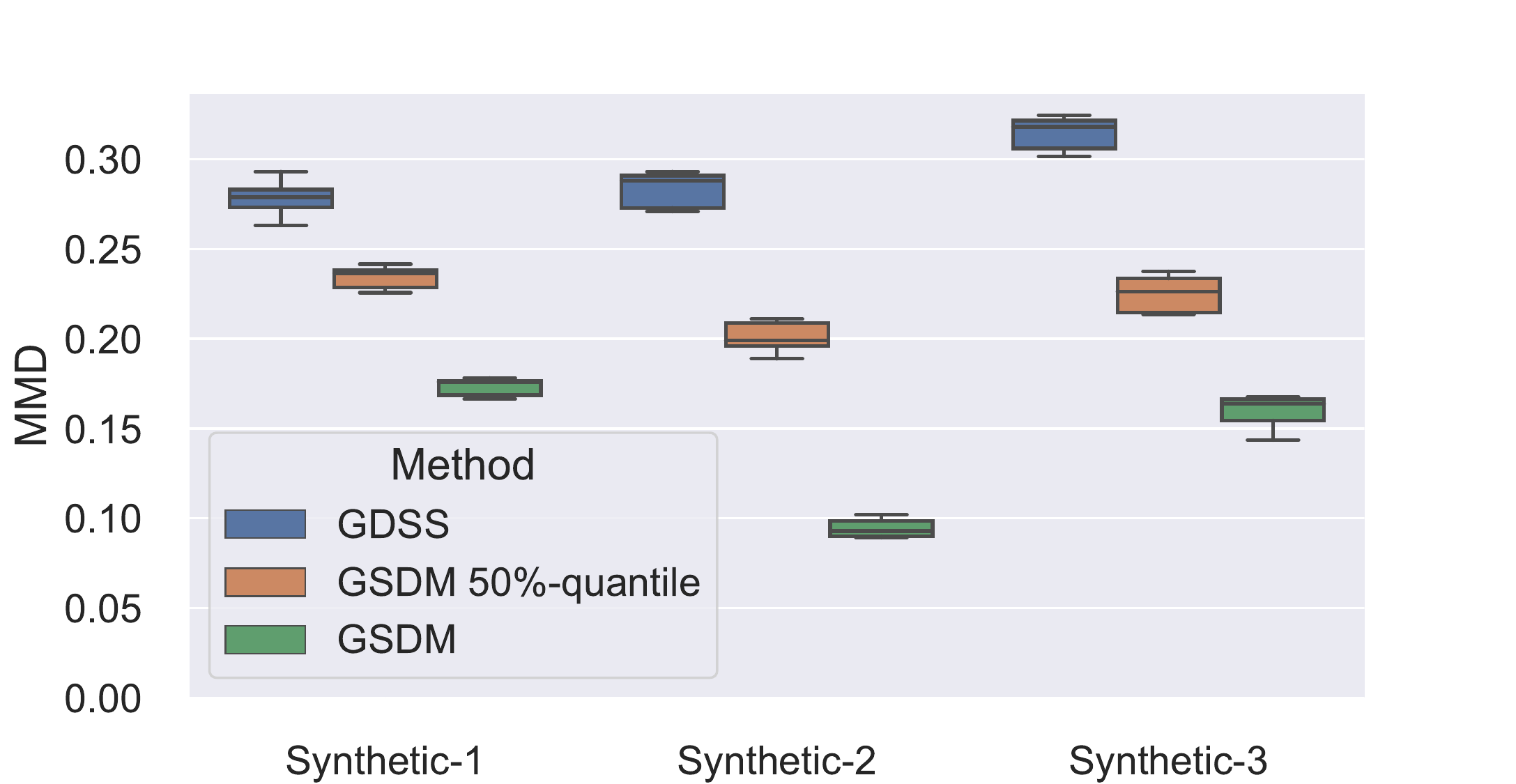}
    \caption{Experiments on synthetic datasets of which the eigenvalues' distributions are shown in Figure \ref{fig:three synthetic}. The y-axis denotes the MMD ($\downarrow$) of the adjacency matrices between the generated graphs to the test graphs. }
    \label{fig:synthetic_results}
\end{figure}

\noindent \textbf{Ablation studies on low-rank approximations of the score-function.} As mentioned in Section 4.2, one can further accelerate the training and sampling phases, by confining diffusion to the partial spectrum of the graph adjacency matrix. Such variants are coined as $\alpha$-quantile GSDM in Section 4.2. In this experiment, we select the top $[\alpha n]$ eigenvalues to conduct the sampling process with a well-trained GSDM, where $\alpha$ is chosen from $\{0.1,0.2,0.3,0.4,0.5,0.6,0.7,0.8,0.9,1.0\}$. The results are shown in Table \ref{tab:eigen}. 

Numerical experiments prove that the low-rank $\alpha$-quantile GSDM retains the state-of-the-art performance among most of the baseline methods. Especially, for the Enzymes dataset, by merely preserving the top-30\% eigenvalues and eigenvectors, the $0.3$-quantile GSDM can still retain 98.2\% of the performance of the full GSDM. The ablation study justifies that our proposed GSDM is capable of capturing essential knowledge that resides on low-dimensional manifolds. It further demonstrates that $\alpha$-quantile GSDM enables outstanding scalability for large-scale datasets.

\noindent \textbf{Ablation studies on graphs with different eigenvalues distributions.} To analyze how GSDM performs under graphs with different eigenvalue distributions, we randomly generate 3 types of synthetic graphs with eigenvalues distributions shown in Figure \ref{fig:three synthetic}: synthetic-1 has evenly distributed eigenvalues, and synthetic-2 and synthetic-2 have moderate and high distinctions in eigenvalues. Due to the predefined eigenvalue distributions, the edge values of the synthetic graphs are not limited to $\{0,1\}$. Thus instead of using the metrics presented in Section 5.1, we directly compute the MMD of the adjacency matrices between the generated graphs and the test graphs. Corresponding results compared to GDSS are shown in Figure \ref{fig:synthetic_results}. We also include an ablation variant of GSDM by using only top-50$\%$ quantile eigenvalues. Both GSDM and its ablation variant outperform the GDSS model under this synthetic setting.




\section{Conclusion}

In this paper, we present a novel graph diffusion model named GSDM, which consists of diffusion methods on both graph node features and topologies through SDEs. Specifically, for graph topology diffusion, we design a novel spectral diffusion model in GSDM, to improve the accuracy of predicting score functions, and avoid the pitfalls of the conventional diffusion processes on graphs. Furthermore, we provide theoretical analysis to justify the advantages in effectiveness and efficiency of GSDM compared to the standard graph diffusion. To validate our proposed model, we conducted extensive experiments showing that our proposed GSDM outperforms the state-of-the-art baseline methods on both generic graph generation and molecule generation with significantly higher processing speed.


\bibliography{ref}
\bibliographystyle{IEEEtran}

\appendices

\section{Proof Details}
\subsection{Proof of Proposition \ref{Prop_spectral_SDE_adj}} \label{proof_prop1}

\begin{proof}
Plugging $\mathbf{A}_t = \mathbf{U}_0 \bm{\Lambda}_t \mathbf{U}_0^\top$ into \eqref{prop1_lam_sde} yields 
\begin{align*}
    \mathrm{d} \mathbf{A}_t
    = &
    -\frac{1}{2} \sigma_{\Lambda,t}^2 
    \mathbf{A}_t
    \mathrm{d}t
    +
    \sigma_{\Lambda,t} \mathrm{d} \mathbf{M}_t,
    \\
    \mathbf{M}_t
    \triangleq &
    \mathbf{U}_0 \mathbf{W}_t^{\Lambda} \mathbf{U}_0^\top
    =
    \mathbf{U}_0 \cdot \mathrm{diag}(\mathbf{B}_t^{\Lambda}) 
    \cdot \mathbf{U}_0^\top.
\end{align*}

Since $(\mathbf{M}_t)_{t\in\R}$ is a linear transformation of $(\mathbf{B}_t^{\Lambda})_{t\in\R}$, the standard Brownian motion on $\R^n$, hence $(\mathbf{M}_t)_{t\in\R}$ is a rank-$n$ centered Gaussian process in $\R^{n\times n}$. Moreover, $\mathbf{M}_t$ is characterized by its covariance kernel $\mathcal{K}(s,t): [0,1]\times [0,1] \mapsto \mathbb{R}^{n\times n \times n\times n}$, which is given by
\begin{small}
\begin{align*}
        & \mathcal{K}(s,t)_{i,j,k,l}
        \\
        & \triangleq
        \mathrm{Cov}(
            \mathbf{M}_s[i,j],
            \mathbf{M}_t[k,l]
        )
        =
        \mathbb{E}[
            \mathbf{M}_s[i,j]
            \mathbf{M}_t[k,l]
        ]
        \\
        & =
        \mathbb{E}
        \left(
            \sum_{h=1}^n
            \mathbf{U}_0[i,h] 
            \mathbf{B}_s^{\Lambda}[h]
            \mathbf{U}_0[j,h]
        \right)
        \left(
            \sum_{h=1}^n
            \mathbf{U}_0[k,h] 
            \mathbf{B}_t^{\Lambda}[h]
            \mathbf{U}_0[l,h]
        \right)
        \\
        & =
        \mathbb{E}
        \left(
            \sum_{h=1}^n
            \mathbf{B}_s^{\Lambda}[h] \mathbf{B}_t^{\Lambda}[h]
            \mathbf{U}_0[i,h] 
            \mathbf{U}_0[j,h] 
            \mathbf{U}_0[k,h] 
            \mathbf{U}_0[l,h] 
        \right)
        \\
        & =
        \sum_{h=1}^n
        \mathbb{E}(\mathbf{B}_s^{\Lambda}[h] \mathbf{B}_t^{\Lambda}[h])
        \mathbf{U}_0[i,h] 
        \mathbf{U}_0[j,h] 
        \mathbf{U}_0[k,h] 
        \mathbf{U}_0[l,h] 
        \\
        & =
        \min(s, t)
        \sum_{h=1}^n
        \mathbf{U}_0[i,h] 
        \mathbf{U}_0[j,h] 
        \mathbf{U}_0[k,h] 
        \mathbf{U}_0[l,h].
    \end{align*}  
\end{small}

 Notice that \eqref{spectral_fwd_sde} is a Ornstein–Uhlenbeck process, which admits a closed-form solution
    \begin{align}
        \bm{\Lambda}_t
        =
        \bm{\Lambda}_0 
        e^{-\frac{1}{2} \int_0^t \sigma_{\tau}^2 \mathrm{d} \tau}
        +
        (1 - e^{- \int_0^t \sigma_{\tau}^2 \mathrm{d} \tau}) \mathbf{W}_1^{\Lambda}.
        \label{solution_Lam_t}
    \end{align}
    Hence, plugging \eqref{solution_Lam_t} into $\mathbf{A}_t=\mathbf{U}_0 \bm{\Lambda}_t \mathbf{U}_0^\top$ yields
    \begin{align}
        \mathbf{A}_t
        =
        \mathbf{A}_0
        e^{-\frac{1}{2} \int_0^t \sigma_{\tau}^2 \mathrm{d} \tau}
        +
        (1 - e^{- \int_0^t \sigma_{\tau}^2 \mathrm{d} \tau}) \mathbf{M}_1.
    \end{align}
This completes the proof.
\end{proof}

\subsection{Proof of Proposition \ref{prop2_error_bound}} \label{proof_prop2}

In preparation of the main proof, we first establish some technical lemmas.

\begin{lemma}[Reconstruction Bound for Generic Diffusion]
\label{lemma_generic_reconstruction_bound}
    We first consider the following oracle reversed time SDE on $(\R^d, \|\cdot\|)$ 
    \begin{equation*}
        \mathrm{d} \bar{\mathbf{Z}}_t
        =
        \left(
            \mathbf{f}(\bar{\mathbf{Z}}_t,t)
            -
            \sigma_t^2
            \nabla_{\mathbf{Z}} 
            \log p_t(\bar{\mathbf{Z}}_t)
        \right)
        \mathrm{d} \bar{t}
        +
        \sigma_t
        \mathrm{d} \bar{\mathbf{B}}_t,\
        t\in [0,1],
    \end{equation*}
    and we define the corresponding estimated reverse time SDE as
    \begin{equation*}
        \mathrm{d} \widehat{\mathbf{Z}}_t
        =
        \left(
            \mathbf{f}(\widehat{\mathbf{Z}}_t,t)
            -
            \sigma_t^2
            s_{\bm{\phi}}(\widehat{\mathbf{Z}}_t, t)
        \right)
        \mathrm{d} \bar{t}
        +
        \sigma_t
        \mathrm{d} \bar{\mathbf{B}}_t,\
        t\in [0,t],
    \end{equation*}
    where $s_{\bm{\phi}}(\cdot)$ is optimized to predict the Stein score function $\nabla_{\mathbf{Z}} \log p_t(\mathbf{Z}_t)$ by minimizing the score matching objective
    \begin{equation*}
        \min_{\bm{\phi}}
        \mathcal{E}(\bm{\phi})
        \triangleq
        \E_{\mathbf{Z}}
        \E_{\mathbf{Z}_t|\mathbf{Z}}
        \|
            s_{\bm{\phi}}(\mathbf{Z}_t, t) 
            - 
            \nabla_{\mathbf{Z}} \log p_{t}(\mathbf{Z}_t)
        \|^2.
    \end{equation*}
    Then for any $\bm{\phi}$, the construction error is bounded by
    \begin{align}
    \label{generic_reconstruction_bound}
        \E \|
            \mathbf{Z}_0 - \widehat{\mathbf{Z}}_0
        \|^2
        \leqslant &
        C^2 
        \|\sigma_{\cdot}\|_{\infty}^4
        \mathcal{E}(\bm{\phi})
        \notag
        \\
        &
        \cdot
        \left(
            1
            +
            \int_0^1
                F(t) 
                \exp
                \left(
                    \int_t^1
                        F(s)
                    \mathrm{d}s
                \right)
            \mathrm{d}t
        \right),
    \end{align}
    where
    $
    F(t) 
    \triangleq 
    C^2
    \sigma_t^4
    \|s_{\bm{\phi}}(\cdot,t) \|_{\mathrm{lip}}^2
    +
    C
    \|\mathbf{f}(\cdot, t)\|_{\mathrm{lip}}^2
    $ and $C$ is a constant.
\end{lemma}

\begin{proof}[Proof of Proposition \ref{prop2_error_bound}]
    Assumptions of Proposition \ref{prop2_error_bound} imply 
    \begin{align*}
        \|\mathbf{f}(\cdot, t)\|_{\mathrm{lip}}
        = &
        \frac{1}{2} \sigma_t^2,
        \\
        \|s_{\bm{\phi}}(\cdot, t)\|_{\mathrm{lip}}
        \leqslant &
        L\cdot \E_{|\mathbf{A}_0}
        \|
        \nabla \log p_{t|0}(\widehat{\mathbf{A}}_t^{\mathrm{full}})
        \|_{\mathrm{lip}},
        \\
        \|s_{\bm{\varphi}}(\cdot, t)\|_{\mathrm{lip}}
        \leqslant &
        L\cdot \E_{|\mathbf{A}_0}
        \|
        \nabla \log p_{t|0}(\widehat{\mathbf{A}}_t^{\mathrm{spec}})
        \|_{\mathrm{lip}},
    \end{align*}
    where $L$ is a constant.
    According to \cite{song_ncsn,song2020score}, at each time step $t\in [0,1]$, the oracle conditional score function $\nabla \log p_{t|0}(\cdot)$ is equivalent to a mapping that maps the input $d$-dimensional random variable $\mathbf{Z}_t$ to $\Sigma_t^{-1} \bm{\varepsilon}$. The shown up $\bm{\varepsilon}$ is a $d$-dimensional standard Gaussian vector that is independent to $\mathbf{Z}_t$, and $\Sigma_t$ is the standard deviation of $\mathbf{Z}_t$, which is given by 
    $
    \Sigma_t
    \triangleq
    \left(
        1 - e^{
            -
            \int_0^t
                \sigma_s^2
            \mathrm{d}s
        }
    \right)^{\frac{1}{2}}.
    $
    Thus, the expected Lipschitz norm of oracle conditional score function is bounded by
    \begin{align}
        \left(
            \E_{|\mathbf{A}_0}
            \|
                \nabla \log p_{t|0}(\mathbf{Z}_t)
            \|_{\mathrm{lip}}
        \right)^2
        & \leqslant 
        \E_{|\mathbf{A}_0}
        \|
            \nabla \log p_{t|0}(\mathbf{Z}_t)
        \|_{\mathrm{lip}}^2 \nonumber
        \\
        & \leqslant
        \frac{\Sigma_t^{-2}}{\|\mathbf{Z}_0\|^2}
        \E
         \|\bm{\varepsilon}\|^2
         =
         \frac{\Sigma_t^{-2}d}{\|\mathbf{Z}_0\|^2}.
         \label{lip_bound}
    \end{align}
    With Lemma \ref{lemma_generic_reconstruction_bound} in hand, we only need to plug \eqref{lip_bound} into \eqref{generic_reconstruction_bound}, by substituting the notations $(\mathbf{Z}, d, \| \cdot\|)$ with $(\widehat{\mathbf{A}}_t^{\mathrm{full}}, n^2, \|\cdot\|_{2,2})$ and $(\widehat{\bm{\Lambda}}_t^{\mathrm{spec}}, n, \|\cdot\|_{2,2})$ to bound the Lipschitz norm of the score networks. This leads us to

\begin{align*}
    & \E\|
    \mathbf{A}_0
    -
    \widehat{\mathbf{A}}_0^{\mathrm{full}}
    \|^2
    \\
    & \leqslant
    M
    \mathcal{E}(\bm{\phi})
    \cdot
    \left(
        1
        +
        n^2 K
        \cdot
        \int_0^1
            \Sigma_t^{-2}
            \exp
            \left(
            n^2 K
            \int_t^1
                \Sigma_s^{-2}
            \mathrm{d}s
            \right)
        \mathrm{d}t
    \right),
    \\
    & \E\|
    \bm{\Lambda}_0
    -
    \widehat{\bm{\Lambda}}_0^{\mathrm{spec}}
    \|^2
    \\
    & \leqslant
    M
    \mathcal{E}(\bm{\phi})
    \cdot
    \left(
        1
        +
        n K
        \cdot
        \int_0^1
            \Sigma_t^{-2}
            \exp
            \left(
            n K
            \int_t^1
                \Sigma_s^{-2}
            \mathrm{d}s
            \right)
        \mathrm{d}t
    \right),
\end{align*}
where $M \triangleq C^2 \|\sigma_{\cdot}\|_{\infty}^4$ and $K\triangleq 2ML/\E
    \|
        \mathbf{A}_0
    \|_{2,2}$. For the spectral diffusion part, the final proof step is completed by the fact that
\begin{align*}
    &\E
    \|
        \mathbf{A}_0
        \!-\!
        \widehat{\mathbf{A}}_0^{\mathrm{spec}}
    \|^2
    \!\!=\! 
    \E
    \|
    \mathbf{U}_0
    \!\left(
        \bm{\Lambda}_0
        \!-\!
        \widehat{\bm{\Lambda}}_0^{\mathrm{spec}}
    \right)\!
    \mathbf{U}_0^\top
    \|^2
    \!=\! 
    \E
    \|
    \bm{\Lambda}_0
    \!-\!
    \widehat{\bm{\Lambda}}_0^{\mathrm{spec}}
    \|^2,
    \\
    &\E
    \|
        \mathbf{A}_0
    \|_{2,2}
    =
    \E
    \|
        \mathbf{U}_0
        \bm{\Lambda}_0
        \mathbf{U}_0^\top
    \|_{2,2}
    =
    \E
    \|
        \bm{\Lambda}_0
    \|_{2,2}.
\end{align*}
This completes the proof.    
\end{proof}

\begin{proof}[Proof of Lemma \ref{lemma_generic_reconstruction_bound}]

To bound the expected reconstruction error $\E\|\mathbf{Z}_0 - \widehat{\mathbf{Z}}_0\|^2$,  we first analysis how $\mathbb{E}\|\bar{\mathbf{Z}}_t - \widehat{\mathbf{Z}}_t\|^2$ evolves as time $t$ is reversed from $1$ to $0$. By definition, it holds that
\begin{small}
    \begin{equation*}
        \bar{\mathbf{Z}}_t \!-\! \widehat{\mathbf{Z}}_t
        \!=\! 
        \int_{1}^t\!\!
            \left(
                \mathbf{f}(\bar{\mathbf{Z}}_s,s)
                \!-\!
                \mathbf{f}(\widehat{\mathbf{Z}}_s,s)
            \right)
            \!+\!
            \sigma_s^2\!
            \left(
                s_{\bm{\phi}}(\widehat{\mathbf{Z}}_{s},s) 
                \!-\! 
                \nabla_{\mathbf{Z}} \log p_s(\bar{\mathbf{Z}}_{s})
            \right)
        \mathrm{d}\bar{s}.
    \end{equation*}
\end{small}
By taking norm and expectation on both sides, we have

\begin{small}
    \begin{align*}
        & \E \|
            \bar{\mathbf{Z}}_t - \widehat{\mathbf{Z}}_t
        \|^2
        \\
        & \leqslant
        \E\!
        \int_{1}^t\!
        \bigg\|\!
            \left(
                \mathbf{f}(\bar{\mathbf{Z}}_s,s)
                \!-\!
                \mathbf{f}(\widehat{\mathbf{Z}}_s,s)
            \right)
            \!+\!
            \sigma_s^2
            \!\bigg(
                s_{\bm{\phi}}(\widehat{\mathbf{Z}}_{s},s) 
                \!-\! 
                \nabla_{\mathbf{Z}} \log p_s(\bar{\mathbf{Z}}_{s})
            \!\bigg)
        \!\bigg\|^2\!
        \mathrm{d}\bar{s}
        \\
        & \leqslant
        C\E
        \int_{1}^t
        \left\|
            \mathbf{f}(\bar{\mathbf{Z}}_s,s)
            -
            \mathbf{f}(\widehat{\mathbf{Z}}_s,s)
        \right\|^2
        \mathrm{d}\bar{s}
        \\
          & \;\;\; +
        C\E
        \int_{1}^t
        \sigma_s^4
        \left\|
            s_{\bm{\phi}}(\widehat{\mathbf{Z}}_{s},s) 
            - 
            \nabla_{\mathbf{Z}} \log p_s(\bar{\mathbf{Z}}_{s}) 
        \right\|^2
        \mathrm{d}\bar{s}
        \\
        \notag
        & \leqslant
        C^2
        \int_{1}^t
        \sigma_s^4
        \cdot
        \E
        \left\|
            s_{\bm{\phi}}(\widehat{\mathbf{Z}}_{s},s) 
            - 
            s_{\bm{\phi}}(\bar{\mathbf{Z}}_{s},s) 
        \right\|^2
        \\
        & \;\;\; +
        \sigma_s^4
        \cdot
        \E
        \left\|
            s_{\bm{\phi}}(\bar{\mathbf{Z}}_{s},s) 
            - 
            \nabla_{\mathbf{Z}} \log p_s(\bar{\mathbf{Z}}_{s})
        \right\|^2
        \mathrm{d}\bar{s}
        \\
        & \;\;\; + C
        \int_{1}^t
        \|\mathbf{f}(\cdot,s)\|_{\mathrm{lip}}^2
        \cdot
        \E
        \left\|
            \bar{\mathbf{Z}}_s
            -
            \widehat{\mathbf{Z}}_s
        \right\|^2
        \mathrm{d}\bar{s}
        \\
        & \leqslant
        \int_{1}^t
        \underbrace{
        \left(
            C^2
            \sigma_s^4
            \|s_{\bm{\phi}}(\cdot,s) \|_{\mathrm{lip}}^2
            +
            C
            \|\mathbf{f}(\cdot, s)\|_{\mathrm{lip}}^2
        \right)
        }_{F(s)}
        \cdot
        \E
        \left\|
            \widehat{\mathbf{Z}}_{s}
            - 
            \bar{\mathbf{Z}}_{s} 
        \right\|^2
        \mathrm{d}\bar{s}
        \\
        & \;\;\; +
        \underbrace{
            C^2 
            \int_{1}^t
                \sigma_s^4
            \mathrm{d}\bar{s}
            \cdot
            \mathcal{E}(\bm{\phi})
        }_{G(t)}.
    \end{align*}
\end{small}

    The proof is completed by applying the Gr\"{o}nwall's inequality to $t\mapsto \mathbb{E}\|\bar{\mathbf{Z}}_t - 
    \widehat{\mathbf{Z}}_t\|^2$, which yields
    
    \begin{align*}
         & \E \|
            {\mathbf{Z}}_0 - \widehat{\mathbf{Z}}_0
        \|^2
        \\
        & \leqslant
        G(0)
        +
        \int_0^1
            G(t) F(t) 
            \exp
            \left(
                \int_t^1
                    F(s)
                \mathrm{d}s
            \right)
        \mathrm{d}t
        \\
        & \leqslant
        C^2 
        \|\sigma_{\cdot}\|_{\infty}^4
        \mathcal{E}(\bm{\phi})
        \cdot
        \left(
            1
            +
            \int_0^1
                F(t) 
                \exp
                \left(
                    \int_t^1
                        F(s)
                    \mathrm{d}s
                \right)
            \mathrm{d}t
        \right).
    \end{align*}
\end{proof}
\subsection{Proof of Proposition~\ref{Prop_sm_converge}}\label{proof_prop3}

\begin{proof}
    Since $\mathcal{E}(\bm{\theta}) = \mathbb{E}_{\mathcal{S}\sim \mathcal{D}^N} \widehat{\mathcal{E}}(\bm{\theta}; \mathcal{S}_N)$, we only need to bound the empirical risk $\widehat{\mathcal{E}}(\bm{\theta}; \mathcal{S}_N)$.
    By assumption, denote
    \begin{align*}
        \mathbf{h}
        \triangleq
        \begin{pmatrix}
            & s_{\bm{\theta}}(\mathbf{Z}^1_0)-\nabla \log p(\mathbf{Z}^1_0|\mathbf{Z}^1_1)^\top
            \\
            & \vdots
            \\
            & s_{\bm{\theta}}(\mathbf{Z}^N_0)-\nabla \log p(\mathbf{Z}^N_0|\mathbf{Z}^N_1)^\top
        \end{pmatrix}
        \in \mathbb{R}^{Nd\times 1},
    \end{align*}
    It holds that
        \begin{equation*}
        \|\nabla_{\bm{\theta}} \widehat{\mathcal{E}}(\bm{\theta}; \mathcal{S}_N) \|^2
        =
        \frac{1}{N^2}
        \mathbf{h}^\top \mathbf{K}_{\bm{\theta}}(\mathcal{S})\mathbf{h}
        \geqslant
        \frac{\lambda}{N^2} \|\mathbf{h}\|^2
        =
        \frac{\lambda}{N} \widehat{\mathcal{E}}(\bm{\theta}; \mathcal{S}_N),
    \end{equation*}
    which implies that $\widehat{\mathcal{E}}(\bm{\theta}; \mathcal{S}_N)$ satisfies the $\frac{\lambda}{N}$-Polyak \L ojasiewicz condition \cite{SGD_convergence}. Then the proof is completed by applying Theorem 7 in \cite{SGD_convergence} to $\widehat{\mathcal{E}}(\bm{\cdot}; \mathcal{S}_N)$.
\end{proof}

\section{Algorithm and Additional Results}
\begin{figure}[h]
\vspace{-0.2in}
\centering
\begin{minipage}{0.9\linewidth}
\centering
\begin{algorithm}[H]
    \small
    \caption{Training GSDM via minimizing score-matching objective}\label{alg:training}
        \textbf{Input:} Score networks $\mathbf{s}_{\bm{\theta},t}(\cdot)$, $\mathbf{s}_{\bm{\phi},t}(\cdot)$, maximal diffusion time $T$, drift functions $\mathbf{f}^X(\cdot,t), \mathbf{f}^{\Lambda}(\cdot,t)$, noise schedules $\sigma_{X,t}, \sigma_{\Lambda,t}$, learning rate $\eta$ and training epochs $K$. \\
        \textbf{Output:} Optimized score network parameters $\bm{\theta}_K, \bm{\phi}_K$.
    \begin{algorithmic}[1]
        \STATE Initialize $\bm{\theta}_0,\bm{\phi}_0$
        \FOR{$k=1$ \textbf{to} $K$}
            \STATE $(\mathbf{X}_0,\mathbf{A}_0) \sim \mathcal{G}$\;
            \STATE $\bm{\Lambda}_0 \gets \mathrm{EigenValues}(\mathbf{A}_0)$\;
            \STATE $t\sim \mathrm{Unif}([0,T])$\;
            \STATE 
            $\mathbf{X}_t 
            \sim
            \int_0^t 
            f^{X}(\mathbf{X}_{\tau},\tau) 
            \mathrm{d} \tau
            +
            \int_0^t
            \sigma_{X,\tau}
            \mathrm{d} \mathbf{B}^X_{\tau}$,\
            $\bm{\Lambda}_t 
            \sim
            \int_0^t 
            f^{\Lambda}(\bm{\Lambda}_{\tau},\tau) 
            \mathrm{d} \tau
            +
            \int_0^t
            \sigma_{\Lambda,\tau}
            \mathrm{d} \mathbf{B}^{\Lambda}_{\tau}$\;
            \STATE
            $\widehat{\mathcal{E}}(\bm{\theta}_k)
            \gets \| s_{\bm{\theta}_k}(\mathbf{X}_t,\bm{\Lambda}_t) - \nabla \log p_{t|0}(\mathbf{X}_t|
            \mathbf{X}_0) \|^2$\;
            \STATE
            $\widehat{\mathcal{E}}(\bm{\phi}_k)
            \gets \| s_{\bm{\phi}_k}(\mathbf{X}_t,\bm{\Lambda}_t) - \nabla \log p_{t|0}(\bm{\Lambda}_t|
            \bm{\Lambda}_0) \|^2$\;
            \STATE
            $(\bm{\theta}_{k+1}, \bm{\phi}_{k+1})
            \gets (\bm{\theta}_{k}, \bm{\phi}_{k})
            -
            \eta (\nabla \widehat{\mathcal{E}}(\bm{\theta}_k), \nabla \widehat{\mathcal{E}}(\bm{\phi}_k))$\;
        \ENDFOR
        \STATE \textbf{Return:} $\bm{\theta}_{K}, \bm{\phi}_{K}$\;
    \end{algorithmic}
    \label{alg_train}
\end{algorithm}
\end{minipage}

\end{figure}

\begin{figure}[h]
\vspace{-0.2in}
\centering
\begin{minipage}{0.9\linewidth}
\centering
\begin{algorithm}[H]
    \small
    \caption{Sampling via GSDM with VP-SDE predictor-corrector solver}
        \textbf{Input:} Score-based models $\mathbf{s}_{\bm{\theta},t}(\cdot)$ and $\mathbf{s}_{\bm{\phi},t}(\cdot)$, maximal diffusion time $T$, number of sampling steps $M$,  Langevin-MCMC step sizes $\{\epsilon_i\}_{i=1}^M$, noise schedules $\{\beta_i\}_{i=1}^M$ and priori distribution $\pi$. \\
        \textbf{Output:} Generated graph data $(\widehat{\mathbf{X}}_0,\widehat{\mathbf{A}}_0)$
    \begin{algorithmic}[1]
        \STATE $t\gets T$\;
        \STATE $(\widehat{\mathbf{X}}_T,\widehat{\bm{\Lambda}}_T)\sim \pi$, $\widehat{\mathbf{U}}_0 \sim \mathrm{Unif}(\{\mathbf{U}\triangleq \mathrm{EigenVectors}(\mathbf{A}), (\mathbf{X}, \mathbf{A}) \sim \mathrm{Data}\})$\;
        \FOR{$m=M-1$ \textbf{to} $0$}
            \STATE $\mathbf{S}_{X} \gets \mathbf{s}_{\bm{\theta},t}(\widehat{\mathbf{X}}_{t},\widehat{\bm{\Lambda}}_{t},\widehat{\mathbf{U}}_0)$, 
            $\mathbf{S}_{\Lambda} \gets \mathbf{s}_{\bm{\phi},t}(\widehat{\mathbf{X}}_{t},\widehat{\bm{\Lambda}}_{t},\widehat{\mathbf{U}}_0)$\;
            \STATE $t' \gets t - T/(2M)$\;
            \STATE $\widehat{\mathbf{X}}_{t'} \gets (2 - \sqrt{1- \beta_{m+1}}\widehat{\mathbf{X}}_{t} + \beta_{m+1} \mathbf{S}_{X}) + \sqrt{\beta_{m+1}} \mathbf{z}_{X}$, $\mathbf{z}_X\sim N\left(\bm{0},\mathbf{I}\right)$ \COMMENT{Prediction step: $\mathbf{X}$}
            \STATE $\widehat{\bm{\Lambda}}_{t'} \gets (2 - \sqrt{1- \beta_{m+1}}\widehat{\bm{\Lambda}}_{t} + \beta_{m+1} \mathbf{S}_{\Lambda}) + \sqrt{\beta_{m+1}}\mathbf{z}_{\Lambda}$, $\mathbf{z}_{\Lambda}\sim N\left(\bm{0},\mathbf{I}\right)$ \;\COMMENT{Prediction step: $\bm{\Lambda}$}
            
            \STATE $\mathbf{S}_{X} \gets \mathbf{s}_{\bm{\theta},t'}(\widehat{\mathbf{X}}_{t'}, \widehat{\bm{\Lambda}}_{t'},\widehat{\mathbf{U}}_0)$,
            $\mathbf{S}_{\Lambda} \gets \mathbf{s}_{\bm{\phi},t'}(\widehat{\mathbf{X}}_{t'}, \widehat{\bm{\Lambda}}_{t'},\widehat{\mathbf{U}}_0)$\; 
            \STATE $t\gets t' - T/(2M)$\;
            \STATE $\widehat{\mathbf{X}}_{t}\gets \widehat{\mathbf{X}}_{t'} + \epsilon_i \mathbf{S}_{X} +  \sqrt{2\epsilon_i}\bm{z}_{X}$, $\mathbf{z}_{X}\sim N\left(\bm{0},\mathbf{I}\right)$\; \COMMENT{Correction step: $\mathbf{X}$}
            \STATE $\widehat{\bm{\Lambda}}_{t} \gets \widehat{\bm{\Lambda}}_{t'} + \epsilon_i \mathbf{S}_{\Lambda} +  \sqrt{2\epsilon_i}\mathbf{z}_{\Lambda}$, $\mathbf{z}_{\Lambda}\sim N\left(\bm{0},\mathbf{I}\right)$\; \COMMENT{Correction step: $\bm{\Lambda}$}
        \ENDFOR
        \STATE $\widehat{\mathbf{A}}_0 = \widehat{\mathbf{U}}_0 \widehat{\bm{\Lambda}}_0 \widehat{\mathbf{U}}_0^\mathbf{T}$\; 
        \STATE \textbf{Return:} $(\widehat{\mathbf{X}}_0,\widehat{\mathbf{A}}_0)$\; 
    \end{algorithmic}
    \label{alg_sample}
\end{algorithm}
\end{minipage}

\end{figure}
\subsection{Algorithm of GSDM }\label{algorithm}

The pseudo codes of training and sampling with GSDM are summarized in Algorithm \ref{alg_train} and Algorithm \ref{alg_sample}. In a nutshell, we first train a GSDM from real data via minimizing score-matching error. On top of that, we are able to generate graph features and eigen-values of graph adjacency matrices by reversing the forward spectral SDE. By uniformly sampling eigen-vectors from training set, we can construct plausible graph adjacency matrix via spectral composition.

\begin{figure}[t!]
    \centering
    \includegraphics[width=0.75\columnwidth]{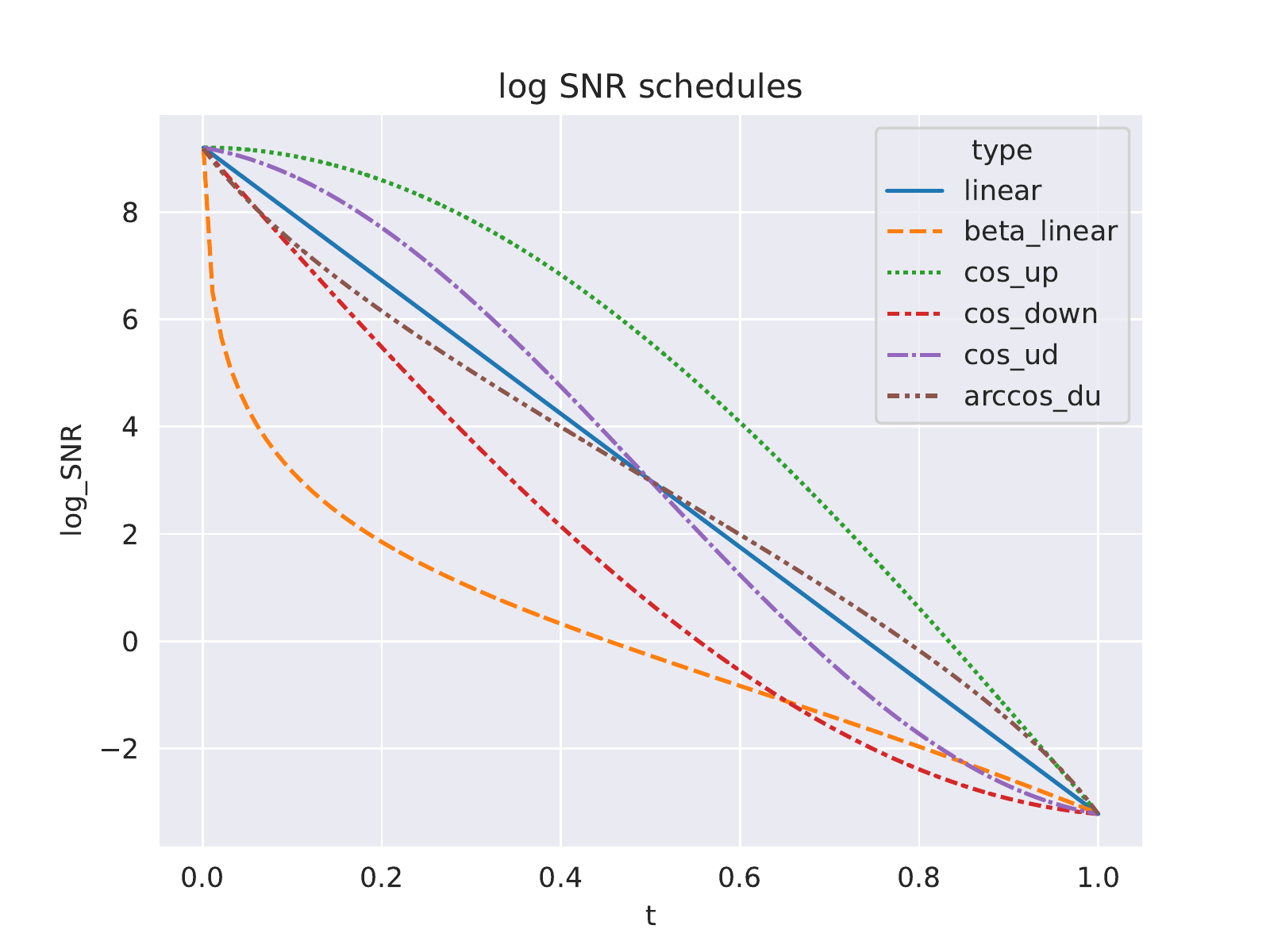}
    \caption{Diffusion schedules.}
    \label{noise_schedule}
\end{figure}
\begin{figure}[t!]
    \centering
    \includegraphics[width=0.45\columnwidth]{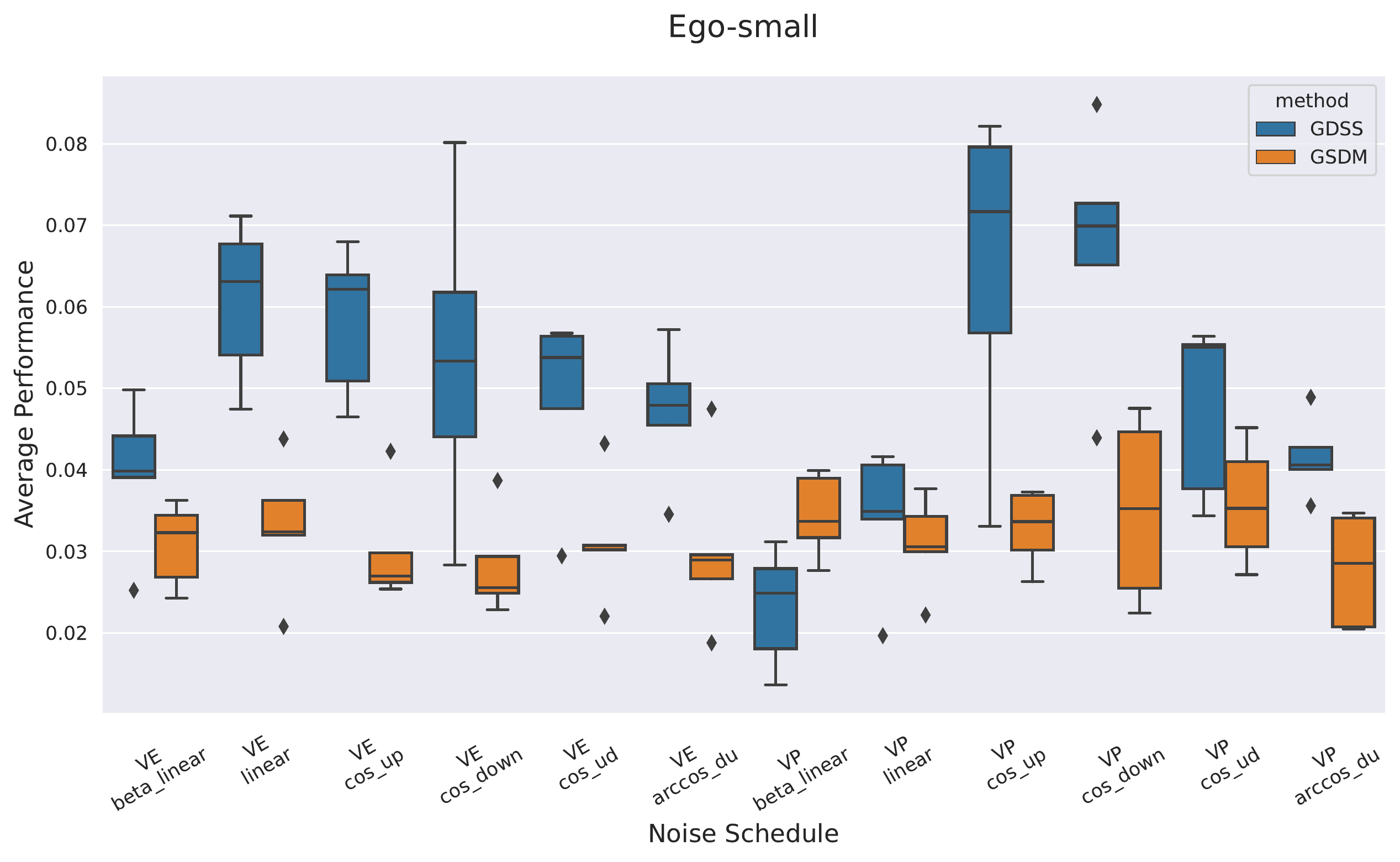}
    \includegraphics[width=0.45\columnwidth]{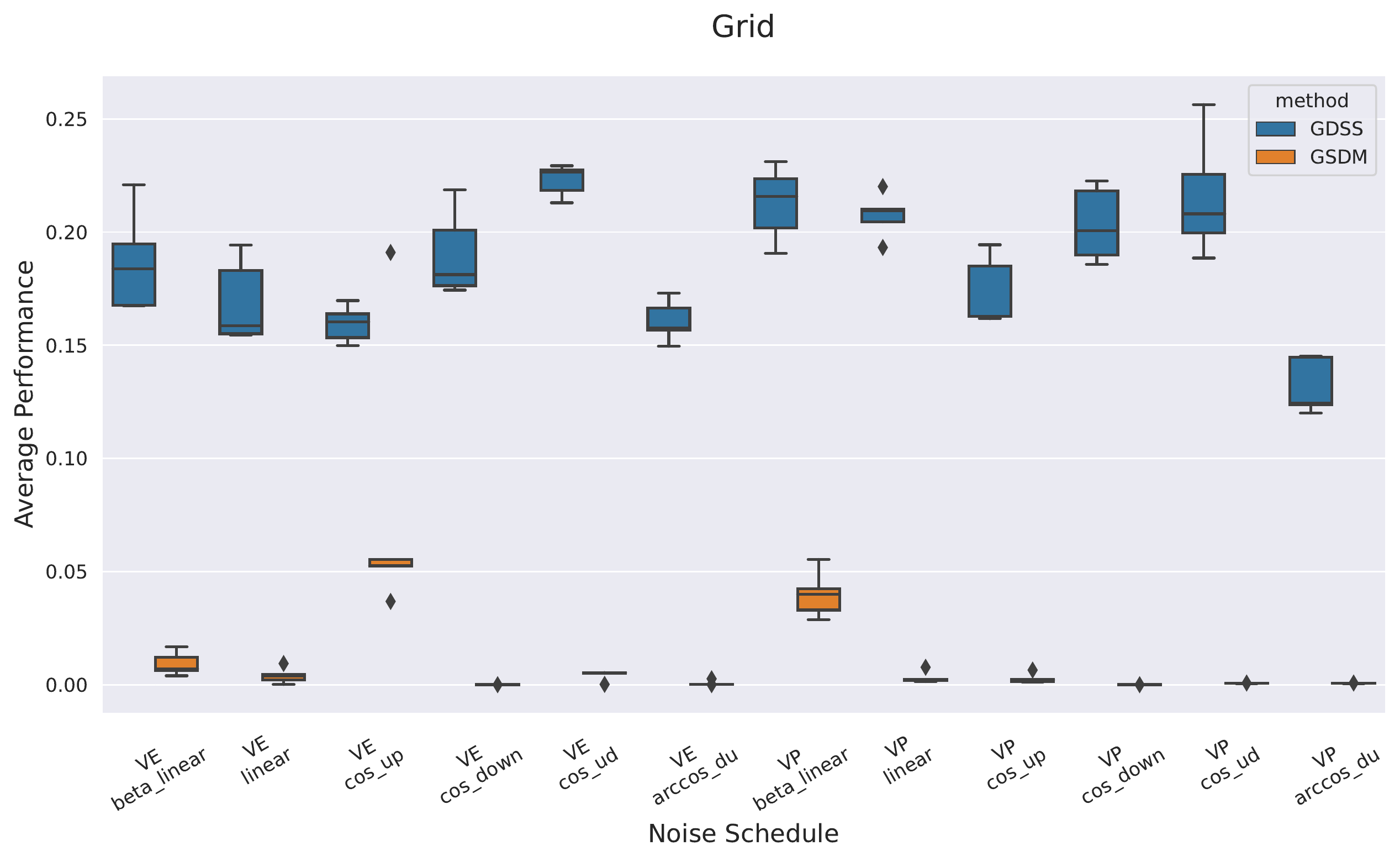}
    \caption{Ablation study on the choice of noise schedules.}
    \label{snr_abalation}
\end{figure}
\subsection{Ablation study on diffusion schedules}\label{diffusion_schedule_ablation}

In this section, we empirically verify the robustness of our proposed GSDM under various types of diffusion schedules. As shown in Figure \ref{noise_schedule}, we design six representative noise schedules with the same initial and final signal-to-noise ratio (defined in \cite{song2020score,improved_ddpm}), while exhibiting different behaviour along the diffusion process. For both VP- and VE-SDE configurations, we train our proposed GSDM with six representative noise schedules and we evaluate the average performance scores. Results on Ego-small and Grid in Figure \ref{snr_abalation} show that GSDM is robust to the choice of diffusion schedules, and it is able to achieve SOTA performance without deliberate noise schedule optimization.

\begin{figure}[h!]
    \centering
    \includegraphics[width=0.98\columnwidth]{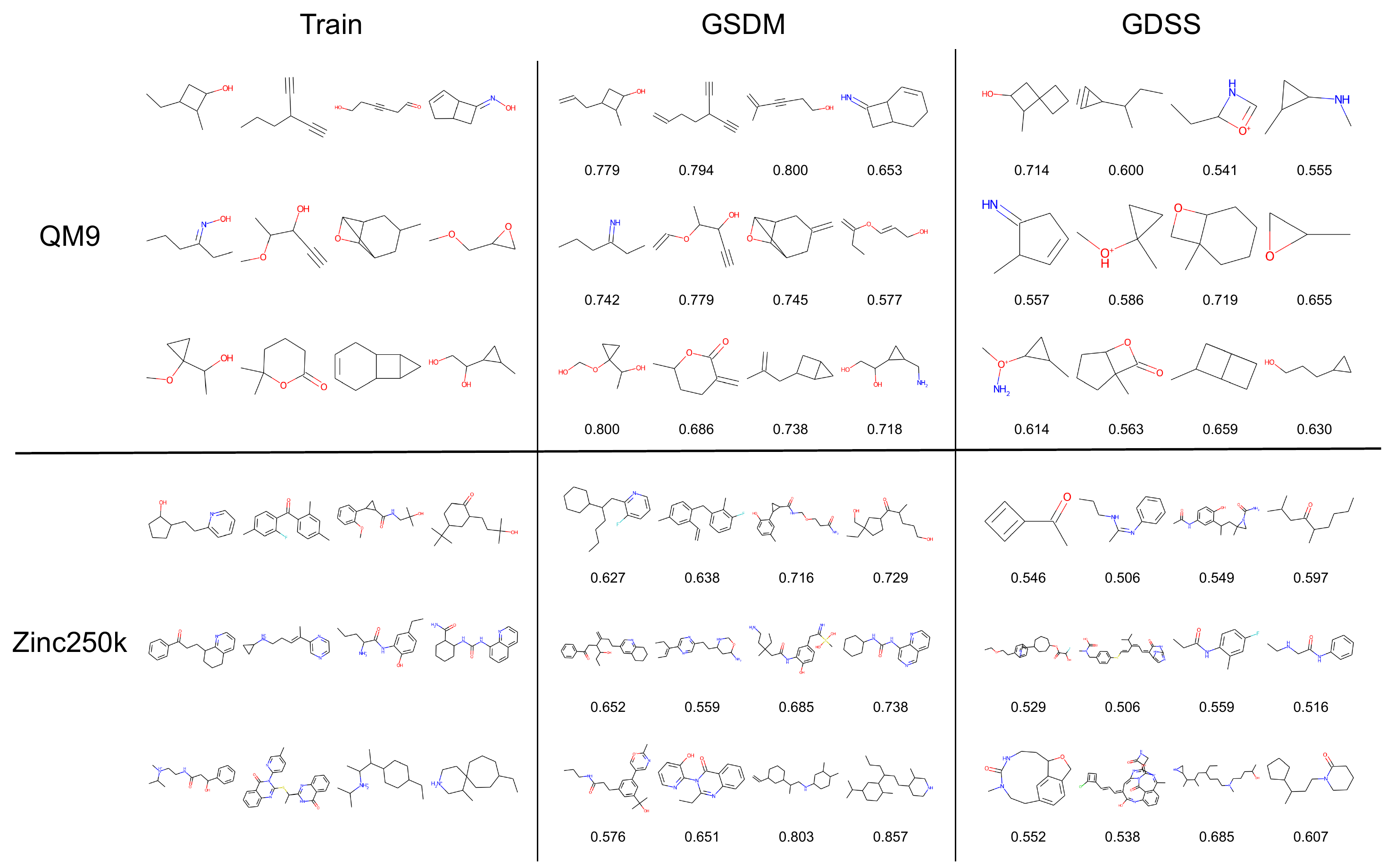}
    \caption{Visualization of molecule generation with maximum Tanimoto similarity of GSDM comparing to GDSS. The left part shows randomly selected molecules from the training set of QM9 and Zinc250k. For each generated molecules, we show the Tanimoto similarity value at the bottom.}
    \label{fig:molecule_visualization}
\end{figure}
\subsection{Visualization of molecule generations}\label{molecule_visualization}

In Figure \ref{fig:molecule_visualization}, we show the generated molecules that are maximally similar to certain training molecules. We compute the molecule similarity using the Tanimoto similarity based on the Morgan fingerprints, which are implemented based on REKit \cite{landrum2016rdkit}. Higher Tanimoto scores indicate that the model is able to generate more similar molecules as the training set, which reflects the learning capability of the model. As shown in the figure, compared to the GDSS model, GSDM is able to generate molecules that have a more similar distribution as the molecules in the training set.

\begin{IEEEbiography}{Tianze Luo} received the Bachelor degree (First Class Honours) and the Master degree from Nanyang Technological University, Singapore, in 2017 and 2019, respectively. He is currently working towards an Alibaba Talent Programme Ph.D. degree with Alibaba-NTU Joint Research Institute, Nanyang Technological University, Singapore. 
\end{IEEEbiography}

\begin{IEEEbiography}{Zhanfeng Mo} is currently a Ph.D. student in the School of Computer Science and Engineering of Nanyang Technological University, Singapore. He received his Bachelor degree in Statistics from University of Science and Technology of China in 2020. His main research interests include theory and algorithms of statistical machine learning.
\end{IEEEbiography}

\begin{IEEEbiography}{Sinno Jialin Pan} is a Provost's Chair Professor with Nanyang Technological University (NTU), Singapore. He received his Ph.D. degree in computer science from the Hong Kong University of Science and Technology (HKUST) in 2011. Prior to joining NTU, he was a scientist and Lab Head of text analytics with the Data Analytics Department, Institute for Infocomm Research, Singapore. He serves as an Associate Editor for IEEE TPAMI, AIJ, and ACM TIST.
\end{IEEEbiography}


\end{document}